\documentclass[11pt]{article}

\usepackage[margin=1.05in]{geometry}
\usepackage{amsmath,amssymb,amsthm,mathtools}
\usepackage{microtype}
\usepackage{natbib}
\usepackage{hyperref}
\usepackage{xcolor}
\usepackage{enumitem}
\usepackage{tikz}
\usetikzlibrary{arrows.meta,positioning,calc}
\usepackage{pgfplots}
\pgfplotsset{compat=1.18}

\hypersetup{
  colorlinks=true,
  linkcolor=blue!60!black,
  citecolor=blue!60!black,
  urlcolor=blue!60!black
}

\newtheorem{theorem}{Theorem}[section]
\newtheorem{lemma}[theorem]{Lemma}
\newtheorem{proposition}[theorem]{Proposition}
\newtheorem{corollary}[theorem]{Corollary}
\theoremstyle{definition}
\newtheorem{definition}[theorem]{Definition}
\newtheorem{remark}[theorem]{Remark}

\newcommand{\eps}{\varepsilon}
\newcommand{\calA}{\mathcal A}
\newcommand{\calC}{\mathcal C}
\newcommand{\calD}{\mathcal D}
\newcommand{\calG}{\mathcal G}
\newcommand{\calH}{\mathcal H}
\newcommand{\calL}{\mathcal L}
\newcommand{\calX}{\mathcal X}

\newcommand{\bbE}{\mathbb E}
\newcommand{\bbP}{\mathbb P}
\newcommand{\one}{\mathbf 1}

\newcommand{\dist}{\operatorname{dist}}
\newcommand{\supp}{\operatorname{supp}}

\newcommand{\clip}{\operatorname{clip}}

\newcommand{\At}{\operatorname{At}}

\newcommand{\TCal}{\operatorname{TCal}}
\newcommand{\MA}{\operatorname{MA}}

\newcounter{algorithmctr}
\newenvironment{algorithm}[1]{%
  \refstepcounter{algorithmctr}%
  \par\medskip\noindent\textbf{Algorithm \thealgorithmctr. #1}\par\smallskip
  \begin{enumerate}[leftmargin=2.5em,label=\arabic*.]
}{%
  \end{enumerate}\medskip
}

\title{Optimal Deterministic Multicalibration and Omniprediction}
\author{Georgy Noarov \\ University of Pennsylvania \and Aaron Roth \\ University of Pennsylvania}
\date{\today}

\begin{document}
\hypersetup{pageanchor=false}
\pagenumbering{gobble}
\maketitle
\vspace{-1em}

\begin{abstract}
A model is multicalibrated on a collection of group weights $\calG$ if it is calibrated---i.e. unbiased even conditional on its prediction---not just overall, but also after reweighting contexts by each $g \in \calG$. It is a useful property for many downstream applications and is a basic desideratum of trustworthy machine learning. Before this work, all algorithms known to attain the minimax-optimal $\widetilde O(\eps^{-3})$ sample complexity rate for $\eps$-multicalibration output \emph{randomized} predictors, while deterministic predictors were known only with substantially worse sample complexity. Whether randomization is \emph{necessary} for optimal sample complexity in multicalibration was explicitly asked by \cite{collina2026sample} and implicitly in several prior works.

We resolve this  open problem by giving a minimax-optimal multicalibration algorithm that outputs a deterministic predictor. We then generalize the algorithm to produce deterministic predictors that satisfy outcome indistinguishability with respect to finite or finitely covered collections of tests. As an application, this also gives deterministic \emph{omnipredictors} and \emph{panpredictors} with optimal sample complexity, resolving open problems posed by \cite{okoroafor2025near} and \cite{balakrishnan2025panprediction}.
\end{abstract}

\clearpage
\tableofcontents
\clearpage
\pagenumbering{arabic}
\hypersetup{pageanchor=true}

\section{Introduction}\label{sec:intro}

\paragraph{Calibration and multicalibration.}
A predictor is \emph{calibrated} if, conditional on the value it predicts, that value equals the expected outcome~\citep{dawid1982well}.
The standard quantitative measure of miscalibration is the \emph{expected calibration error} (ECE) which sums the magnitude of the prediction-conditional bias across all predicted values.  \emph{Multicalibration}~\citep{hebert2018multicalibration} strengthens calibration by requiring it to hold not just marginally but simultaneously after reweighting by every group function in a collection $\calG$.  The \emph{ECE multicalibration error} is then the maximum, over groups $g\in\calG$, of the group-weighted ECE (Definition~\ref{def:mc}).  Since its introduction, multicalibration and closely related notions have found a wide range of applications---from learning predictors that are simultaneously optimal for many loss functions~\citep{gopalan2021omnipredictors,gopalan2023loss,noarov2025high,roth2024forecasting}, to strengthening complexity-theoretic constructions~\citep{dwork2025supersimulators,casacuberta2024complexity}, to low-complexity algorithms for agreement and information aggregation~\citep{collina2025tractable,collina2026collaborative,kearns2026networked}---which motivates a sharp understanding of what it costs to achieve.

\paragraph{Omniprediction.}
A closely connected goal is \emph{omniprediction}~\citep{gopalan2021omnipredictors}: learning a single predictor that can be simultaneously and cheaply post-processed to optimize a wide variety of downstream loss functions in a way that is competitive with some benchmark class of models---avoiding the need to train a separate predictor for each loss function. Multicalibration is one route towards omniprediction: \citet{gopalan2021omnipredictors} show that a predictor with $\eps$-ECE multicalibration error with respect to an appropriately defined class of groups is also an $\eps$-omnipredictor.  Recent work gives more direct routes through outcome indistinguishability (OI)\footnote{The original paper of \citet{dwork2021outcome} defines a hierarchy of outcome-indistinguishability notions and tests, depending on the access given to distinguishers.  Throughout this paper, outcome indistinguishability refers to their  ``one-sample sample-access'' notion, in which a test sees a context, the prediction on that context, and one outcome drawn either from Nature or from the predictor-induced outcome model.}~\citep{gopalan2023loss,garg2024oracle,okoroafor2025near}: indeed, omniprediction can be obtained via online algorithms with regret $\widetilde O(\sqrt{T})$, equivalently average error $\widetilde O(1/\sqrt{T})$ \citep{garg2024oracle,okoroafor2025near,gibbs-tibshirani-omnipredict}, which is impossible for (multi)calibration \citep{qiao2021stronger,dagan2025breaking,collina2026optimal,collina2026sample}, and so the two problems are distinct.

\paragraph{Randomized versus deterministic predictors.}
In the batch setting, a learner sees $n$ i.i.d. samples from an unknown distribution $P$ over context--outcome pairs and must output a predictor whose population multicalibration error with respect to $\calG$ is at most $\eps$.  \cite{collina2026sample} established the minimax optimal sample complexity of $\eps$-multicalibration as $\widetilde\Theta(\eps^{-3})$ (in the regime in which $|\calG|$ lies anywhere between polylog$(1/\eps)$ and poly$(1/\eps)$). Their lower bound holds for all algorithms, randomized or not, but their upper bound follows from online-to-batch reductions that output randomized predictors \citep{gupta2021online,noarov2025high,ghugeimproved}. \cite{gupta2021online}, who gave the first sample-complexity upper bounds for multicalibration via online-to-batch reductions, already noted that the best known bounds were achieved by randomized predictors. \cite{haghtalab2023unifying} emphasized the same point, and \cite{collina2026sample} explicitly asked whether the minimax sample complexity of multicalibration can be achieved by a learner whose output predictor is deterministic. For the closely related problem of multi-distribution learning \citep{peng2024sample,zhang2024optimal}, randomized predictors can be statistically easier to learn than deterministic predictors, and recent work shows that derandomization can also be computationally hard in the fully general setting \citep{larsen2024derandomizing}. Is multicalibration another setting where randomization gives a statistical advantage?

The question is important because prediction-time randomness complicates auditing, reproducibility, and downstream decision-making. It also sits uneasily with multicalibration as a notion of ``trustworthiness'': two identical individuals can receive different predictions from a randomized predictor solely because it flips different coins.

The same issue arises for omniprediction. The sample-optimal omnipredictors of \citet{okoroafor2025near} are obtained by online-to-batch reductions and output randomized predictors, leading them to ask explicitly whether these omnipredictors can be derandomized, or whether randomness is necessary for sample-optimal omniprediction. Similarly, \cite{balakrishnan2025panprediction} ask the same question for \emph{panprediction} (a group conditional notion of omniprediction) for which the same seeming sample complexity gap between deterministic and randomized predictors arises.

Quantitatively, the best prior deterministic rates were substantially worse than the randomized ones.  Translating the deterministic  multicalibration guarantee of \citet[Theorem~4.2]{haghtalab2023unifying} to ECE multicalibration gives a sample complexity of about $\widetilde O(\eps^{-6})$, while the optimal randomized rate is $\widetilde O(\eps^{-3})$ \citep{collina2026sample}.  For omniprediction, \citet{okoroafor2025near} obtain sample-optimal $\widetilde O(d/\eps^2)$  rates for broad bounded-variation loss classes, where $d$ is the relevant complexity of the loss-derived auditor class, but their construction outputs randomized predictors.  Earlier deterministic constructions had much worse dependence on $\eps$---for example, the rate recalled by \citet[Theorem~C.1]{okoroafor2025near} for \citet{gopalan2023loss} scales as $O(d/\eps^4+\eps^{-10}\log(1/\eps))$. \citet{gibbs-tibshirani-omnipredict} give a direct, deterministic, and sample-optimal result for the special case of proper losses---but their construction uses specific structural properties of proper losses and does not extend to the general case.

\begin{center}
    \textit{We resolve all of these questions and show that prediction-time randomization is not necessary \\ for sample-optimal multicalibration, omniprediction, or panprediction, \\or more generally for \emph{outcome indistinguishability}: \\ deterministic predictors can attain the minimax optimal sample complexity for each.}
\end{center}

\subsection{Our results}
Throughout, when stating simplified rates, we assume that the group/constraint/benchmark family has polynomially bounded cardinality: $|\calG|\le\eps^{-\kappa}$ for a fixed constant $\kappa$, and $\widetilde O(\cdot)$ hides polylogarithmic factors in $1/\eps$ and $|\calG|$.  In the body we allow the learner to use randomness \emph{during training}, while its output is always a fixed \emph{deterministic} function.  Appendix~\ref{app:training-derandomization} shows that this remaining training randomness can also be removed, with only logarithmic changes in the sample bounds.

\begin{enumerate}[leftmargin=2em,label=\textbf{\arabic*.}]
\item \emph{Optimal Deterministic Multicalibration (Theorem~\ref{thm:main}).}
We give an algorithm that outputs a deterministic predictor $h$ with ECE multicalibration error at most $\eps$ using
\[
        n = \widetilde O\left(\frac{1/\eps+\log |\calG|}{\eps^2}\right)
\]
samples.  This is $\widetilde O(\eps^{-3})$, matching the minimax optimal rate \citep{collina2026sample}.  We also show how to implement the learning algorithm in polynomial time. 

\item \emph{Optimal Deterministic Outcome Indistinguishability (Theorem~\ref{thm:linear-tests}).}
We then generalize our result to any finite family $\calA$ of bounded outcome-indistinguishability tests $a:\calX\times[0,1]\to[-1,1]$.  The algorithm uses
\[
        \widetilde O\left(\frac{\log|\calA|}{\eps^2}\right)
\]
samples and outputs a deterministic predictor $h$ satisfying
\[
        \max_{a\in\calA}\left|\bbE[a(X,h(X))(h(X)-Y)]\right|\le\eps .
\]
These rates are optimal, since the tight bounds we give for multicalibration and omniprediction are special cases. 

\item \emph{Optimal Deterministic Omniprediction (Theorem~\ref{thm:det-omni} and Corollaries~\ref{cor:cover-det-omni}--\ref{cor:okk-det-omni}).}
Applying the OI theorem to the tests that certify omniprediction gives deterministic versions of the sample-optimal randomized offline omnipredictors of \citet{okoroafor2025near}. Here is the concrete specialization. For a loss class $\calL$ and benchmark class $\calH$, omniprediction reduces to threshold-calibration tests together with multiaccuracy tests for the loss-derived class
\[
        \Delta\calL\circ\calH
        =
        \{x\mapsto \ell(f(x),1)-\ell(f(x),0):\ell\in\calL,\ f\in\calH\}.
\]
If this class is represented to accuracy $O(\eps)$ by a finite auditor class $\calC$---for example, by a finite cover or finite approximate basis---then the sample complexity is
\[
        \widetilde O\left(\frac{\log|\calC|+\log(1/\eps)}{\eps^2}\right)
\]
matching the randomized offline rate up to logarithmic factors. In particular, if this loss-derived class has pseudo-dimension $p$, then deterministic omnipredictors use $\widetilde O((p+\log(1/\eps))/\eps^2)$ samples. In Appendix~\ref{app:panprediction} we extend this result to panprediction.

\end{enumerate}

\subsection{Technical overview}

To understand our algorithm in the special case of multicalibration it is useful to first consider a simple/naive approach to derandomization that does not work. Given a randomized predictor obtaining low multicalibration error, perhaps we could just ``fix its randomness up front'' to make it deterministic. Put aside representation concerns for now (how do we ``fix the randomness'' of continuously many prediction values?) Indeed if the underlying distribution on contexts was continuous (or at least put only tiny probability on any context $x$) then this would succeed.\footnote{In atomless finite-action settings, this intuition goes back to the purification theorem of \citet{dvoretzky1951elimination}, which eliminates randomization in games while preserving finitely many expected payoffs.  Our setting allows atoms in the feature distribution where exact purification can fail; the machinery we develop is  designed to handle this obstruction.} If in a large sample we never see the same context twice, then there is no observable difference between flipping coins at test time vs. pre-flipping them at training time. Thus the only conceptual obstacle to derandomization is distributions that have non-trivial atoms --- i.e. place large mass on particular contexts $x$. To understand the issue, which arises already for marginal calibration before imposing any group structure, consider the following minimal example.
\paragraph{A two-point obstruction.}
Consider a distribution over two contexts $x_1$ and $x_2$, both of which have probability $1/2$. Suppose $x_1$ is always paired with label $y = 0$ and $x_2$ is always paired with label $y = 1$.  Consider the randomized predictor that on input $x_1$ outputs prediction $\tilde y = 1/3$ with probability $2/3$ and otherwise outputs prediction $\tilde y = 2/3$, and that on input $x_2$ outputs prediction $\tilde y = 2/3$ with probability $2/3$ and otherwise outputs prediction $\tilde y = 1/3$.
 
This predictor has ECE zero: the expected bias of the prediction conditional on predicting $\tilde y = 1/3$ is
\[
        \frac12\cdot\frac23\cdot(1/3-0)
        +
        \frac12\cdot\frac13\cdot(1/3-1)
        =0,
\]
and the analogous bias conditional on predicting $\tilde y = 2/3$ is also zero.  But any rounding that assigns each context a prediction from the support of its randomized prediction distribution produces a predictor $h$ with $h(x_i)\in\{1/3,2/3\}$. Any such predictor has ECE at least $1/6$.  The randomized predictor is calibrated because conditioning on each prediction value gives exactly the right mixture of the two contexts, and this mixture is destroyed by deterministic rounding.  Thus any rounding-based derandomization must explicitly account for atoms.

\paragraph{A first attempt: handle large atoms separately.}
If the distribution has no atoms, or if all atoms have very small mass, then fixing the predictor's randomness up front works: an atom $x$ of mass $p_x$ contributes variance on the scale of $p_x^2$ to the realized multicalibration error. As all $p_x$ tend to $0$, the multicalibration error is unaffected by fixing the randomness.  At the other extreme, if an atom has sufficiently large mass that it appears many times in the training data, we can estimate its conditional label mean directly. This would solve the two-context example above, because each of $x_1$ and $x_2$ accounts for roughly half of the sample. These two observations suggest a hybrid algorithm: fix the randomness of a randomized multicalibrated predictor on ``light''/infrequent contexts, and separately estimate the conditional label mean of each sufficiently frequent ``heavy'' context.

Unfortunately the numbers do not work out. 
Let $L$ denote the log of the number of signed calibration tests---informally, the exponent in the union bound used to control multicalibration error after fixing the randomness of a randomized predictor. $L$ is on the order of the cardinality of the predictor's domain, so $L\approx 1/\eps$ up to logarithmic factors. Suppose an atom is declared to be ``heavy'' if it has mass more than $\tau$, and is otherwise declared to be ``light''. For any fixed signed calibration test, the variance proxy for the rounding noise from the light atoms is bounded by a constant times $\sum_x p_x^2$. Since every light atom has mass $p_x\le\tau$, this satisfies $\sum_x p_x^2\le \tau\sum_x p_x\le\tau$.  Simultaneous concentration over the $L$ signed calibration tests then asks for $\sqrt{L\tau}\lesssim\eps$ by a union bound, and so the threshold for ``light'' atoms that guarantees that they do not contribute more than $\eps$ additional multicalibration error is roughly
\[
        \tau_{\rm light}\approx \frac{\eps^2}{L},
\]
which is about $\eps^3$.  On the other hand, directly estimating the conditional label mean of an atom to constant multiples of $\eps$ requires about $\eps^{-2}$ occurrences of that atom.  If we have $N$ samples, an atom of mass $p_x$ appears about $Np_x$ times, so direct estimation is only guaranteed when
\[
        Np_x\gtrsim \eps^{-2},
        \qquad\text{or equivalently}\qquad
        p_x\gtrsim \tau_{\rm heavy}(N):=\frac{1}{N\eps^2}.
\]
For a hard threshold $\tau$ to handle every atom, we would need $\tau_{\rm heavy}(N)\lesssim\tau \lesssim \tau_{\rm light}$, which requires
\[
        N\gtrsim \frac{L}{\eps^4}.
\]
Thus the hard split approach works only at a sample size much larger than the target rate $N\approx L/\eps^2$.  At that target rate, $\tau_{\rm heavy}(N)$ is about $1/L$, which is about $\eps$, leaving a large range of atom masses between $\eps^3$ and $\eps$ that are too large for rounding but too small for conditional mean estimation.  Figure~\ref{fig:atoms} illustrates this mismatch and the confidence-interval approach that avoids it.

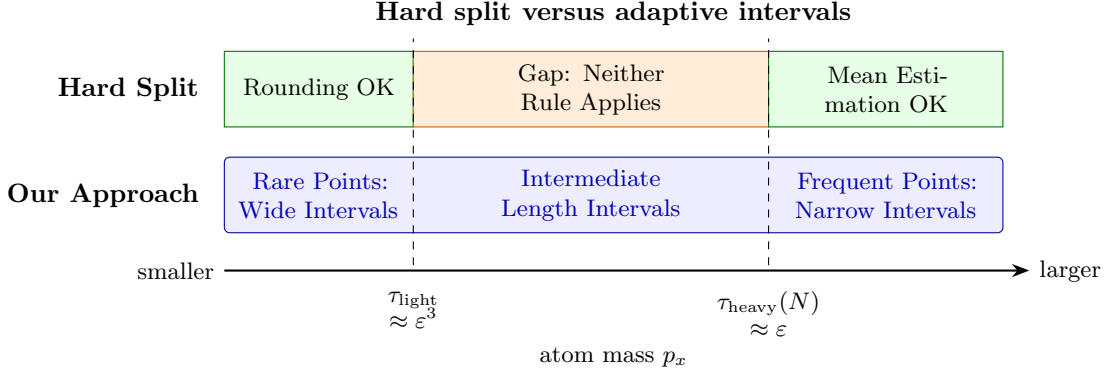
\begin{figure}[ht]
\centering
\begin{tikzpicture}[x=1cm,y=1cm,>=Stealth,font=\footnotesize]
\def\xL{0.85}
\def\xA{3.35}
\def\xB{8.05}
\def\xR{11.15}

\node[font=\small\bfseries] at (6.0,3.75) {Hard split versus adaptive intervals};

\node[anchor=east,font=\small\bfseries] at (\xL-0.2,2.75) {Hard Split};
\draw[fill=green!10,draw=green!45!black] (\xL,2.25) rectangle (\xA,3.25);
\draw[fill=orange!15,draw=orange!75!black] (\xA,2.25) rectangle (\xB,3.25);
\draw[fill=green!10,draw=green!45!black] (\xB,2.25) rectangle (\xR,3.25);
\node[text width=2.0cm,align=center] at (2.10,2.75) {Rounding OK};
\node[text width=3.6cm,align=center] at (5.70,2.75) {Gap: Neither Rule Applies};
\node[text width=2.3cm,align=center] at (9.60,2.75) {Mean Estimation OK};

\node[anchor=east,font=\small\bfseries] at (\xL-0.2,1.35) {Our Approach};
\draw[fill=blue!7,draw=blue!60!black,rounded corners=2pt] (\xL,0.85) rectangle (\xR,1.85);
\node[text width=2.1cm,align=center,blue!70!black] at (2.10,1.35) {Rare Points:\\Wide Intervals};
\node[text width=3.6cm,align=center,blue!70!black] at (5.70,1.35) {Intermediate Length Intervals};
\node[text width=2.4cm,align=center,blue!70!black] at (9.60,1.35) {Frequent Points:\\Narrow Intervals};

\draw[dashed] (\xA,0.42)--(\xA,3.42);
\draw[dashed] (\xB,0.42)--(\xB,3.42);
\draw[->,thick] (\xL,0.35) -- (\xR+0.35,0.35);
\node[below=4pt,align=center] at (\xA,0.35) {$\tau_{\rm light}$\\[-1pt]$\approx\eps^3$};
\node[below=4pt,align=center] at (\xB,0.35) {$\tau_{\rm heavy}(N)$\\[-1pt]$\approx\eps$};
\node[below=25pt] at (6.0,0.35) {atom mass $p_x$ };
\node[anchor=east] at (\xL,0.35) {smaller};
\node[anchor=west] at (\xR+0.35,0.35) {larger};
\end{tikzpicture}
\caption{The obstruction to a hard split.  The gap region contains atom masses that are too large for blind rounding but too small for accurate conditional-mean estimation at the target sample size.  Confidence intervals interpolate between these two regimes.}
\label{fig:atoms}
\end{figure}

\paragraph{Our solution.} Our solution avoids the hard split by instead smoothly incorporating however much statistical information we have about the conditional label mean of a point $x$ as a ``hint'' when producing the randomized predictor that we ultimately round---rather than either relying on the information we have as sufficient or else ignoring it entirely. From the training sample we record how many times each context $x$ appears, and build a confidence interval for its conditional label mean whose width depends on this count $N_x$: frequent contexts receive narrow intervals, while rare contexts receive the trivial interval $[0,1]$.  To use these confidence intervals, we develop a new online learning primitive: an online multicalibration algorithm that receives with each context $x$ a valid interval hint, consisting of an interval $I_x$ containing the conditional label mean and a set of allowed grid values near $I_x$. The algorithm guarantees multicalibration error at the optimal rate against any adversary, provided the hints are valid, and it only makes predictions using the allowed grid values. We then use an online-to-batch reduction using this online learning algorithm paired with confidence intervals computed from a sample of the data. This results in a randomized predictor, but one with context conditional prediction variance that scales inversely with the mass of each point (since frequently seen contexts are endowed with narrower confidence intervals which constrain the support of the randomized predictor). The variance from fixing the randomness of this predictor therefore scales smoothly with the mass of each context, rather than depending on a brittle heavy/light threshold. Finally, we use the data to partition the context space into finitely many rounding cells and use one sampler seed per cell, giving us a finite derandomization even for continuous context spaces. We now map our results in more detail.

\paragraph{Multicalibration upper bound.}
The learner splits its sample into three independent parts: a confidence sample $S_0$ used to learn confidence intervals for the conditional mean around each context, an online-learning sample $S_1$ used in the online-to-batch reduction, and a partition sample $S_2$ used to define the cells used for deterministic rounding. The upper-bound proof has four ingredients.

\begin{enumerate}[leftmargin=2em,label=(\arabic*)]
\item \emph{Online-to-batch reduction with interval hints} (Section~\ref{sec:intervalhints}).  Consider an online adversarial learning setting in which, after the adversary reveals context $x$, the learner also receives a valid interval hint: an interval $I_x$ containing the conditional mean $\mu(x)$ and allowed grid values near that interval.  A minimax and exponential-weights argument gives an online learning algorithm (Lemma~\ref{lem:hinted-online}) that obtains multicalibration against any such adversary at the optimal rate, while predicting only allowed grid values. With valid interval hints in the batch setting, a standard martingale online-to-batch reduction for multicalibration \citep{gupta2021online} then produces a randomized batch predictor with small ECE multicalibration error whose support at each context is restricted to those allowed grid values.
\item \emph{Building the intervals} (Section~\ref{sec:intervals}).  The confidence sample $S_0$ is used to compute the confidence intervals $I_x$.  Atoms seen repeatedly get narrow empirical confidence intervals; unseen and infrequently seen contexts get the trivial interval $[0,1]$.  
\item \emph{Making the rounding finite} (Section~\ref{sec:partition}).  The final step will fix one sampler seed for each rounding cell.  The contexts observed in $S_0$ already give finitely many singleton cells.  The remaining issue is the unobserved region, which may be infinite. We need to partition it into a finite number of ``rounding cells'' whose sum of squared masses is small (so that they will not substantially harm multicalibration error even when rounded with a common seed in the next step). In one dimension, the natural version of this idea is to draw a fresh sample of points, sort them, and use the gaps between consecutive sample points as cells; an exchangeability argument then shows that two fresh points are unlikely to land in the same gap. We use the same rank-based idea for general contexts by imposing a lexicographic order, sorting the partition sample $S_2$ in that order, and taking the cells induced by adjacent sampled contexts.

\item \emph{Rounding the Predictor} (Section~\ref{sec:rounding}).  Using one sampler seed per rounding cell turns the randomized predictor produced by the online-to-batch reduction on $S_1$ into a deterministic function.  The rounding itself is simple: each rounding cell gets one shared random draw from the randomized predictor, and every context in that cell is rounded using that shared draw.  The proof is driven by a general finite-test rounding proposition: if the learned intervals make the mass-squared weighted interval radii small on observed atoms and the partition makes $\sum_C P_X(C)^2$ small on the unobserved region, then one sampler seed per cell preserves every test in any finite family.  This is why the same derandomization step applies both to multicalibration and to outcome indistinguishability.
\end{enumerate}

Figure~\ref{fig:pipeline} summarizes how the three sample parts combine into the deterministic predictor.

\begin{figure}[ht]
\centering
\begin{tikzpicture}[
  >=Stealth, font=\small,
  blk/.style={draw,rounded corners,align=center,inner sep=4pt,text width=27mm,minimum height=10mm},
  sm/.style={draw,rounded corners,align=center,inner sep=3pt,minimum width=9mm},
  every path/.style={thick}]
\node[sm,fill=gray!12] (S) at (0,0) {i.i.d.\ \\ sample $S$};
\node[sm,fill=blue!10] (S0) at (3,1.7) {$S_0$};
\node[sm,fill=blue!10] (S1) at (3,0) {$S_1$};
\node[sm,fill=blue!10] (S2) at (3,-1.7) {$S_2$};
\draw[->] (S) -- (S0);
\draw[->] (S) -- (S1);
\draw[->] (S) -- (S2);
\node[blk,fill=green!10] (I) at (6.6,1.7) {intervals $I_x$};
\node[blk,fill=green!10] (Q) at (6.6,0) {randomized $Q$ \\ (online-to-batch)};
\node[blk,fill=green!10] (P) at (6.6,-1.7) {lexicographic \\ cells $\Pi$};
\draw[->] (S0) -- (I);
\draw[->] (S1) -- (Q);
\draw[->] (S2) -- (P);
\draw[->,densely dashed] (I) -- node[right,font=\scriptsize,align=left]{constrains \\ support} (Q);
\node[blk,fill=orange!18,font=\bfseries] (h) at (10.8,0) {deterministic \\ predictor $h$};
\draw[->] (Q) -- node[above,font=\scriptsize]{round} (h);
\draw[->] (P) -- node[below=5pt,font=\scriptsize,pos=0.55,fill=white,inner sep=1pt]{rounding-cell seed} (h);
\end{tikzpicture}
\caption{The learner splits its sample into three independent parts.  The confidence sample $S_0$ produces valid interval hints $I_x$ and their allowed grid values; the online-learning sample $S_1$ produces a randomized predictor $Q$ supported on those allowed values (Theorem~\ref{thm:constrained-online}); the partition sample $S_2$ produces a finite family of lexicographic rounding cells $\Pi$ (Section~\ref{sec:partition}).  Drawing one independent sampler seed per rounding cell turns $Q$ into the deterministic predictor $h$ (Lemma~\ref{lem:one-seed-rounding}).}
\label{fig:pipeline}
\end{figure}
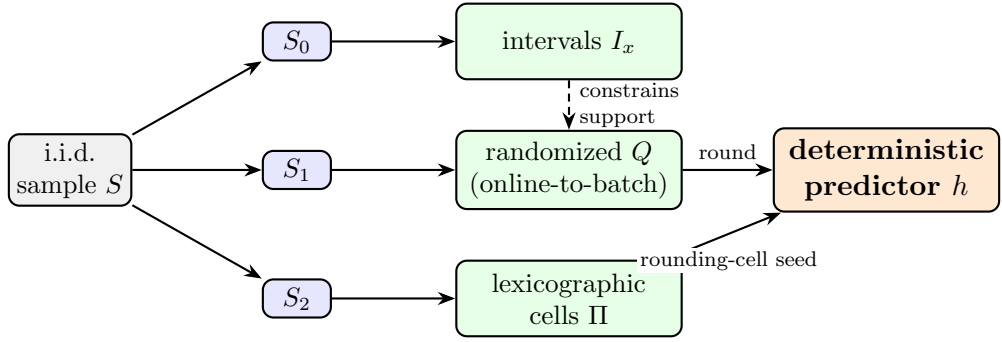

Our online multicalibration algorithm is most naturally described as an exponential-weights subroutine over an exponentially large space, with one weight per \emph{sign pattern} over possible prediction values.  In this form it would take exponential time.  We show, however, that the exponential-weights distribution factors across prediction values, so the algorithm can be implemented implicitly in polynomial time (Theorem~\ref{thm:main}).

\paragraph{Outcome indistinguishability and omniprediction.}
The scalar multicalibration proof uses the special form of ECE multicalibration only through the finite family of signed OI tests that it induces. This lets us state a more general theorem: fix any finite family $\calA$ of bounded OI tests $a(x,v)$, and ask that
\[
        \left|\bbE[a(X,h(X))(h(X)-Y)]\right|
\]
be small for every $a\in\calA$.  Our deterministic learner for this finite-test OI problem uses
\[
        \widetilde O\left(\frac{\log|\calA|}{\eps^2}\right)
\]
samples to make all these correlations at most $\eps$. Omniprediction fits this template because the outcome-indistinguishability characterizations \citep{gopalan2023loss,okoroafor2025near} reduce it to threshold-calibration tests together with multiaccuracy tests for a loss-derived auditor class. Instantiating our theorem with these tests yields sample-optimal omniprediction with a deterministic predictor.

\subsection{Additional Related Work}

\paragraph{Uniform convergence for fixed predictor classes.}
A related but distinct line of work \citep{shabat2020uniform,rosenberg2022exploration} studies uniform convergence of multicalibration error over a fixed predictor class. Models found via ERM over a fixed model class can be deterministic, but these results only control the gap between empirical and population multicalibration error uniformly over a class of candidate predictors---they do not by themselves imply that the class contains any low-error predictor, whereas our results study the sample complexity of producing a predictor with small population multicalibration error.

\paragraph{Multicalibration upper and lower bounds.}
\citet{noarov2025high,collina2026optimal} establish the optimal rate for multicalibration in the online adversarial setting, and \cite{collina2026sample} establish the optimal sample complexity in the batch setting (the setting we study here). The upper bound in \cite{collina2026sample} is randomized, and they ask whether the minimax optimal sample complexity is obtainable via a deterministic predictor. We resolve this question.

Before this work, the strongest deterministic guarantees had worse dependence on the target ECE error.  \citet{haghtalab2023unifying} give deterministic algorithms for bucketed multicalibration; translating their bucketed $L_\infty$ guarantee to ECE requires roughly $1/\eps$ buckets and bucket bias about $\eps^2$, leading to a sample complexity of about $\widetilde O(\eps^{-6})$.  Older deterministic constructions, such as the original multicalibration algorithm of \citet{hebert2018multicalibration}, give still weaker ECE rates after the same translation.  

\paragraph{Multi-distribution learning and derandomization.}
A closely related derandomization question has been studied for agnostic multi-distribution learning.  \citet{larsen2024derandomizing} show that, in the fully general multi-distribution setting, converting randomized multi-distribution predictors into deterministic ones can be computationally hard even when ERM over the benchmark class is efficient.  They also give a positive black-box derandomization under a label-consistency condition, where all distributions share the same conditional label law given the context.  Our setting is complementary: multicalibration, outcome indistinguishability, and omniprediction impose many residual tests or group reweightings, but all are evaluated under a single joint distribution and hence a single conditional mean function $\mu(x)$.  Thus the obstruction exploited in their hardness result is absent, while the central challenge is statistical and distributional: controlling the effect of fixing the prediction-time randomness, especially on atoms.

\paragraph{Omniprediction and outcome indistinguishability.}
Omniprediction was introduced by \citet{gopalan2021omnipredictors} as a way to learn a single predictor whose loss-specific postprocessings compete with a benchmark class across many losses. Their original construction proceeds through multicalibration. Subsequent work developed a more direct outcome-indistinguishability viewpoint (see \cite{dwork2021outcome}): \citet{gopalan2023loss} introduced Loss OI and showed how to decompose it into calibration plus multiaccuracy for a loss-derived class, while \citet{gopalan2023swap} characterized swap variants of omniprediction through swap multicalibration. \citet{garg2024oracle} studied oracle-efficient online multicalibration and omniprediction. 

\citet{hu2022metric} study the sample complexity of outcome indistinguishability, giving metric-entropy characterizations in the distribution-specific setting and fat-shattering characterizations in the distribution-free setting.  Their distinguishers are ``no-access'' distinguishers: they see the context and outcome, but not the prediction value.  Equivalently, their tests are context-only residual tests, which coincide with multiaccuracy.  Our finite-test OI theorem includes this case, but is aimed at the more general prediction-dependent tests needed for calibration and omniprediction. 

\citet{hu2023comparative} study a complementary setting, characterizing distribution-free realizable multiaccuracy and multicalibration by a mutual fat-shattering dimension of the source and distinguisher classes.  Their results exploit realizable structure in the conditional label mean, whereas our results are agnostic in the label distribution and focus on derandomizing optimal-rate multicalibration and OI guarantees.

\citet{okoroafor2025near} gave near-sample-optimal online and offline algorithms for omniprediction, with sample complexity scaling as $\widetilde O(\eps^{-2})$ times the relevant auditor-class complexity. Their offline constructions output randomized predictors through online-to-batch style conversions, and they ask whether these randomized omnipredictors can be derandomized without losing sample optimality. We resolve this question affirmatively.

The deterministic landscape before our work was mixed.  The Loss-OI construction of \citet{gopalan2023loss} outputs deterministic predictors, but its general sample complexity has substantially worse $\eps$-dependence; \citet[Theorem~C.1]{okoroafor2025near} record the bound as $O(d/\eps^4+\eps^{-10}\log(1/\eps))$ for the relevant complexity parameter $d$.  Thus it did not match the optimal $\widetilde O(d/\eps^2)$ rate achieved by randomized online-to-batch methods.  The proper-loss setting is an important exception: \citet{gibbs-tibshirani-omnipredict} give a direct  algorithm outputting a deterministic predictor with the optimal proper-loss omniprediction rate.  Their result exploits the special structure of proper losses, whereas our OI theorem gives a derandomization route for the broader finite-test and finitely covered loss-derived auditor classes used in \citet{okoroafor2025near}.

\paragraph{Panprediction.}
\citet{balakrishnan2025panprediction} introduce panprediction, which generalizes omniprediction by requiring guarantees simultaneously across downstream losses and subgroups. They reduce panprediction to step calibration and obtain deterministic and randomized panpredictors at $\widetilde O(\eps^{-3})$ and $\widetilde O(\eps^{-2})$ sample complexity rates, respectively, and ask whether the gap is inherent. After their reduction, the relevant objectives are again OI tests, now indexed by prediction thresholds, comparator thresholds, and groups. The OI extension we give in Section~\ref{sec:linear-tests} therefore applies, resolving their question of whether optimal panprediction bounds can be obtained with deterministic predictors; Appendix~\ref{app:panprediction} gives the formal statement.

\section{Setting and calibration definitions}\label{sec:setting}

\subsection{Distributions, atoms, and conditional means}

Fix a finite dimension $d<\infty$ and a context space $\calX\subseteq[0,1]^d$.  We consider distributions $P$ over pairs $(X,Y)\in\calX\times[0,1]$.  Let $P_X$ denote the marginal distribution of $X$, and let $\mu(x) = \bbE[Y\mid X=x]$ 
denote the conditional label mean. 

For $x\in\calX$, write
$ p_x = P_X(\{x\})$.
Let
$\At(P_X)=\{x\in\calX: p_x>0\}$ 
denote the set of atoms.  This set is at most countable for every probability measure, and
$\sum_{x\in\At(P_X)}p_x\le1.$ Indeed, for each $j\ge1$ there are at most $2^j$ atoms with mass at least $2^{-j}$, and $\At(P_X)$ is the union of these finite sets.  Points with $p_x=0$ never affect population calibration quantities individually, but the learning procedure still defines predictions on them.

A group is a nonnegative weight function on contexts.  A group family is a finite nonempty set $\calG\subseteq[0,1]^{\calX}$ 
of such functions.   All logarithms are natural.  We suppress universal numerical constants, but all choices below can be made by taking a single sufficiently large universal constant $C$.

\subsection{Grid predictors and ECE multicalibration}

Fix a finite grid $\Lambda\subseteq[0,1]$.  A randomized grid predictor assigns to each context a distribution over grid values:
$Q:\calX\to\Delta(\Lambda),$ 
where $Q_x(v)$ denotes the probability assigned to $v\in\Lambda$.  A deterministic predictor is a function $h:\calX\to\Lambda$, identified with the point-mass randomized predictor $Q_x=\delta_{h(x)}$.

\begin{definition}[Calibration bias at a prediction value]
For a randomized grid predictor $Q$, a group $g\in\calG$, and a grid value $v\in\Lambda$, define
\[
        B_P(Q;g,v)
        =
        \bbE\big[g(X)Q_X(v)(v-Y)\big].
\]
This is the calibration bias contributed by group $g$ among predictions that place mass on value $v$.
Since $Q_X(v)$ and $v$ are functions of $X$,
\[
        B_P(Q;g,v)
        =
        \bbE\big[g(X)Q_X(v)(v-\mu(X))\big].
\]
For a deterministic predictor $h$, this becomes
\[
        B_P(h;g,v)
        =
        \bbE\big[g(X)\one\{h(X)=v\}(v-Y)\big].
\]
\end{definition}

\begin{definition}[ECE multicalibration error]\label{def:mc}
The ECE multicalibration error of $Q$ with respect to $\calG$ is
\[
        \operatorname{MC}_P(Q;\calG)
        =
        \max_{g\in\calG}\sum_{v\in\Lambda}|B_P(Q;g,v)|.
\]
It is sometimes convenient to think about multicalibration error in its dual ``signed'' form:
\end{definition}

\begin{lemma}[Signed form of ECE multicalibration]\label{lem:signed-form}
For every randomized grid predictor $Q$,
\[
        \operatorname{MC}_P(Q;\calG)
        =
        \max_{g\in\calG}\max_{\sigma\in\{\pm1\}^{\Lambda}}
        \bbE\left[g(X)\sum_{v\in\Lambda}\sigma(v)Q_X(v)(v-Y)\right].
\]
For deterministic $h$, this specializes to:
\[
        \operatorname{MC}_P(h;\calG)
        =
        \max_{g\in\calG}\max_{\sigma\in\{\pm1\}^{\Lambda}}
        \bbE\left[g(X)\sigma(h(X))(h(X)-Y)\right].
\]
\end{lemma}
The identity is the elementary fact that $\sum_v |b_v|$ is the largest signed sum $\sum_v\sigma(v)b_v$.

\section{Learning a randomized predictor from valid interval hints}
\label{sec:intervalhints}

We first prove the randomized learning guarantee that the derandomization argument builds on.  In this section the learner is handed context-dependent \emph{interval hints}: intervals $U_x$ together with allowed grid sets $\Lambda_x$.  This is exactly the information that the confidence sample $S_0$ will later supply (Section~\ref{sec:intervals}): the learner does not know $\mu(x)$, but on the coverage event it has an interval $U_x$ containing $\mu(x)$ and a grid set $\Lambda_x$ that approximates every value in $U_x$.  The theorem says that valid interval hints are enough to learn a randomized batch predictor with small ECE multicalibration error whose support stays near $U_x$.  The algorithm behind the theorem is an online multicalibration algorithm, run through the standard multicalibration online-to-batch reduction~\citep{gupta2021online}.

An \emph{interval-hint system} for a finite grid $\Lambda$ assigns to each context $x$ a nonempty closed interval $U_x=[a_x,b_x]\subseteq[0,1]$ and a nonempty set $\Lambda_x\subseteq\Lambda$ of allowed grid values.  Call an interval-hint system \emph{$\gamma$-valid} for $P$ if, for every $x$,
\[
        \mu(x)\in U_x
        \quad\text{and}\quad
        \forall m\in U_x\ \exists a\in\Lambda_x\text{ with }|a-m|\le\gamma .
\]
In the final algorithm the hint system is random because it is learned from the confidence sample $S_0$.  We apply the theorem conditionally on the high-probability event that the learned hints are $\gamma$-valid.

\begin{theorem}[Learning from valid interval hints]\label{thm:constrained-online}
There is a universal constant $C_{\rm on}$ such that the following holds.  Let $\calG$ be a finite nonempty group family, let $\gamma\in(0,1)$, let $\delta\in(0,1)$, and let $\Lambda\subseteq[0,1]$ be a $\gamma$-net of $[0,1]$ with $|\Lambda|\le C_{\rm grid}/\gamma$ for a universal constant $C_{\rm grid}$.  Suppose the learner is given a $\gamma$-valid interval-hint system $(U_x,\Lambda_x)_{x\in\calX}$ on this grid.
Then there is a learning algorithm which, given $T$ i.i.d. samples from $P$ and the interval-hint system, outputs a randomized grid predictor $Q$ satisfying
\[
        \supp(Q_x)\subseteq\Lambda_x\qquad\text{for every }x,
\]
and, with probability at least $1-\delta$,
\[
        \operatorname{MC}_P(Q;\calG)
        \le
        C_{\rm on}\left(\gamma+\sqrt{\frac{1/\gamma+\log|\calG|+\log(1/\delta)}{T}}\right).
\]
\end{theorem}

We next spell out the one-step action rule used by the online algorithm.  After the past history has fixed the current exponential-weights distribution, each context $x$ and grid value $v$ receives a coefficient $c(v)$.  This coefficient is the current aggregate signed ``weight'' for predicting $v$ at $x$.  The learner wants a distribution $q$ over the allowed grid values $\Lambda_x$ whose signed payoff is small for every mean value that could be consistent with the interval hint $U_x=[a_x,b_x]$ (but not necessarily with mean values outside of the interval hint).

For a fixed $q$, the payoff is affine in the unknown mean $m\in U_x$, so the worst case over the interval occurs at one of the two endpoints.  Thus the one-step prediction distribution can be computed by the following finite linear program.  For a context $x$ and coefficients $c:\Lambda\to[-1,1]$, let $\operatorname{LP}_x(c)$ denote the deterministic tie-broken solution $q\in\Delta(\Lambda)$ of
\begin{equation}\label{eq:endpoint-lp}
\begin{array}{ll}
\text{minimize} & \lambda\\
\text{subject to}
    & \sum_{v\in\Lambda}q(v)=1,\qquad q(v)\ge0\quad(v\in\Lambda),\\[0.6ex]
    & q(v)=0\quad(v\notin\Lambda_x),\\[0.6ex]
    & -1\le\lambda\le1,\\[0.6ex]
    & \sum_{v\in\Lambda}q(v)c(v)(v-a_x)\le\lambda,\\[0.6ex]
    & \sum_{v\in\Lambda}q(v)c(v)(v-b_x)\le\lambda .
\end{array}
\end{equation}
The constraints $q(v)=0$ for $v\notin\Lambda_x$ enforce the interval hint's allowed grid set, and $\lambda$ upper-bounds the two endpoint payoffs.

The interval-hint subroutine has two equivalent implementations.  Algorithm~\ref{alg:known-interval-ew} is the direct exponential-weights version used in the proof.  Algorithm~\ref{alg:known-interval-factored} computes the same one-step objectives without enumerating $\{\pm1\}^{\Lambda}$; Appendix~\ref{app:efficient} gives the short factorization argument.

\noindent
\begin{minipage}[t]{0.48\textwidth}
\footnotesize
\begin{algorithm}{Enumerating tests}\label{alg:known-interval-ew}
    \item Set $\mathcal T=\calG\times\{\pm1\}^{\Lambda}$, $M=|\mathcal T|$, and $\eta=\sqrt{(\log M+\log(3/\delta))/T}$.
    \item Initialize $\pi_1$ uniformly on $\mathcal T$.
    \item For $t=1,\ldots,T$:
    \begin{enumerate}[leftmargin=1.5em,label=(\alph*)]
        \item For every query context $x$ and grid value $v\in\Lambda$, set
        \[
            c_t(x,v)=\sum_{r=(g,\sigma)\in\mathcal T}\pi_t(r)g(x)\sigma(v).
        \]
        \item Define $q_t(x)=\operatorname{LP}_x(c_t(x,\cdot))$ for every $x$.
        \item On the realized sample $(X_t,Y_t)$, set
        \[
            z_t^r=g(X_t)\sum_{v\in\Lambda} q_t(X_t)(v)\sigma(v)(v-Y_t)
        \]
        for each $r=(g,\sigma)$.
        \item Update $\pi_{t+1}(r)\propto\pi_t(r)\exp(\eta z_t^r)$.
    \end{enumerate}
    \item Output $Q_x=\frac1T\sum_{t=1}^Tq_t(x)$.
\end{algorithm}
\end{minipage}\hfill
\begin{minipage}[t]{0.48\textwidth}
\footnotesize
\begin{algorithm}{Factored implementation}\label{alg:known-interval-factored}
    \item Set $\eta=\sqrt{(\log(|\calG|2^{|\Lambda|})+\log(3/\delta))/T}$.
    \item Initialize $S_{g,v}^0=0$ for all $g\in\calG$ and $v\in\Lambda$.
    \item For $t=1,\ldots,T$:
    \begin{enumerate}[leftmargin=1.5em,label=(\alph*)]
        \item Set
        \[
            \pi_t(g)\propto\prod_{v\in\Lambda}2\cosh(\eta S_{g,v}^{t-1}).
        \]
        \item For every query context $x$ and $v\in\Lambda$, set
        \[
            c_t(x,v)=\sum_{g\in\calG}\pi_t(g)g(x)\tanh(\eta S_{g,v}^{t-1}).
        \]
        \item Define $q_t(x)=\operatorname{LP}_x(c_t(x,\cdot))$ for every $x$.
        \item On $(X_t,Y_t)$, update
        \[
            S_{g,v}^t=S_{g,v}^{t-1}+g(X_t)q_t(X_t)(v)(v-Y_t),
        \]
        with $q_t(X_t)(v)=0$ for $v\notin\Lambda_{X_t}$.
    \end{enumerate}
    \item Store the states $\{S^{t-1}_{g,v}\}_{t=1}^T$ and output the implicit predictor $Q_x=\frac1T\sum_tq_t(x)$.
\end{algorithm}
\end{minipage}

\begin{lemma}[Online learning with interval hints]\label{lem:hinted-online}
Fix a finite group family $\calG$, a finite grid $\Lambda$ with $K=|\Lambda|$, interval hints $U_x$, allowed grid sets $\Lambda_x$, a covering radius $\gamma$, and $\delta\in(0,1)$.  Let
\[
        \mathcal T=\calG\times\{\pm1\}^{\Lambda},
        \qquad
        M=|\mathcal T|=|\calG|2^K .
\]
Consider any possibly adaptive sequence of rounds with the following structure.  Before round $t$, the past history is fixed; then a context $X_t$ is revealed, the learner chooses $q_t(X_t)\in\Delta(\Lambda_{X_t})$, and then a label $Y_t\in[0,1]$ is drawn with conditional mean
\[
        m_t=\bbE[Y_t\mid X_1,Y_1,\ldots,X_{t-1},Y_{t-1},X_t]\in U_{X_t}.
\]
Assume that, on every realized context, $\Lambda_{X_t}$ $\gamma$-covers $U_{X_t}$.  When  Algorithm~\ref{alg:known-interval-ew} is run,  $q_t(x)$ is supported on $\Lambda_x$ for every $t,x$, and with probability at least $1-\delta/3$,
\[
        \max_{r=(g,\sigma)\in\mathcal T}
        \frac1T\sum_{t=1}^T
        g(X_t)\sum_{v\in\Lambda}q_t(X_t)(v)\sigma(v)(v-Y_t)
        \le
        \gamma+C\sqrt{\frac{\log M+\log(3/\delta)}{T}} .
\]
\end{lemma}

\begin{proof}
The analysis follows a standard reduction from multiobjective optimization to no regret learning (as in e.g. \cite{lee2022online,noarov2025high}) in this case instantiated with exponential weights. The difference is that the minimax step restricts the learner's strategy space to the subset of grid points captured by the interval hints. This constructively restricts the support of the learner's distribution, and does not change the value of the game (and hence the convergence analysis of the algorithm) so long as the interval hints are valid.  Appendix~\ref{app:online-to-batch} gives the details.
\end{proof}

The next lemma is the online-to-batch conversion.  Running the online procedure on i.i.d. samples from a fixed distribution and averaging its iterates gives a randomized batch predictor with comparable population multicalibration error.

\begin{lemma}[Martingale online-to-batch conversion]\label{lem:martingale-otb}
Fix a finite group family $\calG$, a finite grid $\Lambda$, and $\delta\in(0,1)$.  Let
\[
        \mathcal T=\calG\times\{\pm1\}^{\Lambda},
        \qquad
        M=|\mathcal T|.
\]
Let $(X_t,Y_t)_{t=1}^T$ be i.i.d.\ samples from $P$.  Suppose that, before observing $(X_t,Y_t)$, an online procedure chooses a grid predictor $q_t:\calX\to\Delta(\Lambda)$ as a function of the previous samples.  Define
\[
        Q_x(v)=\frac1T\sum_{t=1}^T q_t(x)(v),
        \qquad v\in\Lambda .
\]
For each signed calibration test $r=(g,\sigma)\in\mathcal T$, let
\[
        z_t^r=
        g(X_t)\sum_{v\in\Lambda}q_t(X_t)(v)\sigma(v)(v-Y_t)
\]
and
\[
        A_r(S)=
        \bbE_{(X,Y)\sim P}\left[
        g(X)\sum_{v\in\Lambda}\sigma(v)Q_X(v)(v-Y)
        \right].
\]
Then, with probability at least $1-\delta/3$,
\[
        \max_{r\in\mathcal T}
        \left|
        A_r(S)-\frac1T\sum_{t=1}^Tz_t^r
        \right|
        \le
        C\sqrt{\frac{\log M+\log(3/\delta)}{T}} .
\]
\end{lemma}

\begin{proof}
For a fixed signed test, the population residual of the averaged predictor differs from its online empirical average by a martingale difference sequence: at time $t$, the predictor $q_t$ is fixed before the fresh sample $(X_t,Y_t)$ is drawn.  Azuma--Hoeffding controls this difference for one signed test, and a union bound controls all signed tests at once.  Appendix~\ref{app:online-to-batch} records the details.
\end{proof}
The pieces combine to prove Theorem~\ref{thm:constrained-online}: Lemma~\ref{lem:hinted-online} controls the empirical signed residuals of the online iterates, Lemma~\ref{lem:martingale-otb} transfers this control to the population residuals of the averaged predictor $Q$, and Lemma~\ref{lem:signed-form} converts signed residual control into ECE multicalibration.  Appendix~\ref{app:online-to-batch} spells out the constants and the support guarantee.

\section{Learning the interval hints from repeated contexts}\label{sec:intervals}

This section constructs, from the confidence sample $S_0$, the interval-hint system $(I_x,\Lambda_x)$ used by Theorem~\ref{thm:constrained-online}.  The idea is simple: contexts seen repeatedly get empirical confidence intervals, while contexts seen zero or one times receive the full interval $[0,1]$.  The key point will be that an atom $x$ with a confidence interval of width $r_x$ will eventually contribute to the rounding variance on the order of $r_x^2p_x^2$, which we show is controlled---because wide intervals correspond to infrequent contexts and vice versa.

Fix an internal accuracy parameter $\alpha\in(0,1/10)$.  Let $\gamma=\Theta(\alpha)$, let $\Lambda$ be a $\gamma$-net of $[0,1]$ containing $0$ and $1$, and write $K=|\Lambda|=O(1/\gamma)=O(1/\alpha)$.  Set $L=K+\log(|\calG|+1)+100$, and let $J=\Theta(\log(L/\alpha))$ be the logarithmic confidence parameter.  The confidence sample has size
\[
        n_0=C_0\frac{LJ}{\alpha^2}
\]
for a sufficiently large universal constant $C_0$.  Appendix~\ref{app:interval-hints} records one concrete choice of constants and grid.

Let $O_0=\{X_i:i\le n_0\}$ be the finite set of contexts observed in $S_0$.  For $x\in O_0$, let $N_x=\#\{i\le n_0:X_i=x\}$ and $\widehat\mu_x=N_x^{-1}\sum_{i:X_i=x}Y_i$.  Set
\[
        r_x=\min\left\{1,\sqrt{\frac{J}{N_x}}\right\},
        \qquad
        I_x=[\widehat\mu_x-r_x,\widehat\mu_x+r_x]\cap[0,1].
\]
For $x\notin O_0$, set $N_x=0$, $r_x=1$, and $I_x=[0,1]$.
For every $x\in\calX$, define the allowed grid set
\[
        \Lambda_x=\{v\in\Lambda:\dist(v,I_x)\le\gamma\}.
\]
Together, $(I_x,\Lambda_x)$ form the learned interval-hint system that will be passed to the learning algorithm of Theorem~\ref{thm:constrained-online}.  Proposition~\ref{prop:learned-intervals} records the two facts needed later: the hints are valid, and the atom-weighted radius budget is small.

\begin{proposition}[Learned interval hints]\label{prop:learned-intervals}
There are events $\mathcal E_{\rm cov}$ and $\mathcal E_{\rm rad}$, depending only on $S_0$, with probabilities at least $0.99$ and $0.90$, respectively, such that:
\begin{enumerate}[label=(\roman*)]
    \item on $\mathcal E_{\rm cov}$, the learned interval-hint system $(I_x,\Lambda_x)_{x\in\calX}$ is $\gamma$-valid for $P$;
    \item on $\mathcal E_{\rm rad}$,
    \[
        \sum_{x\in\At(P_X)}p_x^2r_x^2
        \le
        C_1\frac{\alpha^2}{L}
    \]
    for a universal constant $C_1$.
\end{enumerate}
Consequently, with probability at least $0.89$, both conclusions hold.
\end{proposition}

\begin{proof}
We sketch the proof here; the full concentration details are deferred to Appendix~\ref{app:interval-hints}.  With probability one, no non-atom appears twice in $S_0$.  Thus every context with $N_x\ge2$ is an atom, and Hoeffding's inequality controls $|\widehat\mu_x-\mu(x)|$ for each such context.  A union bound over the at most $n_0$ repeated contexts gives the coverage event $\mathcal E_{\rm cov}$.  Contexts with $N_x=0$ or $N_x=1$ receive the full interval $[0,1]$, and the $\gamma$-net property of $\Lambda$ gives the allowed grid value near every point of $I_x$.

For the radius bound, fix an atom $x$ of mass $p$ and write $N_x\sim\operatorname{Bin}(n_0,p)$.  The elementary binomial estimate
\[
        \bbE[p^2r_x^2]\le C\frac{J}{n_0}p
\]
captures the desired tradeoff: either the atom is rare, in which case the factor $p^2$ is already small, or it is frequent, in which case $N_x$ is typically of order $n_0p$ and $r_x^2$ is of order $J/(n_0p)$.  Summing over atoms gives
\[
        \bbE\sum_{x\in\At(P_X)}p_x^2r_x^2
        \le
        C\frac{J}{n_0}
        \le
        C\frac{\alpha^2}{L},
\]
and Markov's inequality gives $\mathcal E_{\rm rad}$.
\end{proof}

\section{Building the finite list of rounding cells}\label{sec:partition}

The final rounding step in Section~\ref{sec:rounding} will use one independent seed per ``rounding cell'', so we need a finite list of such cells.  The partition sample builds such a list on the part of the space not already observed in $S_0$.  The goal is to cut this unobserved region into cells whose summed squared masses are small; this is the term that will control the part of the rounding error not already controlled by the confidence intervals on observed atoms.

Use the lexicographic order inherited from $[0,1]^d$: for distinct $x,z\in\calX$, write $x\prec z$ if, at the first coordinate where they differ, $x$ has the smaller coordinate.  Let the partition sample be $S_2=(W_1,\ldots,W_m)$, an independent sample of contexts from $P_X$, where
\[
        m=C_0\frac{L}{\alpha^2}.
\]
Sort the distinct values appearing in $S_2$:
\[
        w_{(1)}\prec w_{(2)}\prec\cdots\prec w_{(r)} .
\]
Let $U_0=\calX\setminus O_0$ be the region unobserved in $S_0$.  Define $\Pi_{\rm unobs}$ to be the finite family of rounding cells induced by these comparisons, intersected with $U_0$:
\[
\begin{aligned}
        &U_0\cap\{x:x\prec w_{(1)}\},\\
        &U_0\cap\{w_{(j)}\}\quad(1\le j\le r),\\
        &U_0\cap\{x:w_{(j)}\prec x\prec w_{(j+1)}\}\quad(1\le j<r),\\
        &U_0\cap\{x:w_{(r)}\prec x\}.
\end{aligned}
\]

The full rounding-cell family is
\[
        \Pi
        =
        \bigl\{\{x\}:x\in O_0\bigr\}\cup \Pi_{\rm unobs}.
\]

The goal is to show that this partition of the unobserved space does not result in any cell of high mass. The next lemma shows via exchangeability that two distinct draws from the distribution are unlikely to lie in the same cell. 

\begin{lemma}\label{lem:lex-adjacency}
Let $A,B,W_1,\ldots,W_m$ be i.i.d. draws from $P_X$.  Then
\[
        \bbP\left(
        A\ne B
        \text{ and no }W_i\text{ lies in the closed lexicographic interval between }A\text{ and }B
        \right)
        \le
        \frac{2}{m+2}.
\]
\end{lemma}

\begin{proof}
Condition on the multiset of the $m+2$ sampled contexts, and sort its distinct values in lexicographic order.  Call an ordered pair of sampled labels isolated if the two labels have distinct contexts and no other sampled context lies in the closed lexicographic interval between them.  Such a pair can only come from two adjacent distinct context values in the sorted list, and both of those values must appear with multiplicity one.  Hence there are at most $2(m+1)$ isolated ordered pairs.

By exchangeability, conditional on the multiset, the ordered pair of labels corresponding to $(A,B)$ is uniformly distributed over the $(m+2)(m+1)$ ordered pairs of distinct sampled labels.  Therefore the conditional probability that $(A,B)$ is isolated is at most
\[
        \frac{2(m+1)}{(m+2)(m+1)}
        =
        \frac{2}{m+2}.
\]
Averaging over the multiset proves the lemma.
\end{proof}

Finally this lets us control the sum of squared masses of the partition of the unseen context space, which is what we will need to control the rounding error. 

\begin{lemma}[The partition sample controls squared cell masses]\label{lem:partition}
There is an event $\mathcal E_{\rm part}$ with probability at least $0.95$ over $(S_0,S_2)$ such that
\[
        \sum_{C\in\Pi_{\rm unobs}}P_X(C)^2
        \le
        C_2\frac{\alpha^2}{L}
\]
for a universal constant $C_2$.
\end{lemma}

\begin{proof}
Fix $S_0$ and let $\nu(A)=P_X(A\cap U_0)$.  Conditional on $S_0$,
\[
        \sum_{C\in\Pi_{\rm unobs}}P_X(C)^2
        =
        \sum_{C\in\Pi_{\rm unobs}}\nu(C)^2
\]
is the $\nu\times\nu$ mass of pairs that fall in the same unobserved cell.  The contribution from pairs in which the two draws are exactly the same atom is
\[
        \sum_{x\in\At(P_X)\cap U_0}p_x^2.
\]
For distinct $x\prec x'$, the two points can fall in the same cell only if no partition-sample cutpoint lies in the closed lexicographic interval
\[
        [x,x']_\prec=\{z\in\calX:x\preceq z\preceq x'\}.
\]
Since $\nu\le P_X$, the off-diagonal contribution is at most the probability of the following event: draw independent $A,B,W_1,\ldots,W_m\sim P_X$, have $A\ne B$, and have no $W_i$ in the closed lexicographic interval between $A$ and $B$.
By Lemma~\ref{lem:lex-adjacency}, this probability is at most $2/(m+2)$.  Therefore
\begin{equation}\label{eq:partition-conditional}
        \bbE_{S_2}\left[
        \sum_{C\in\Pi_{\rm unobs}}P_X(C)^2
        \,\middle|\, S_0
        \right]
        \le
        \sum_{x\in\At(P_X)\cap U_0}p_x^2+\frac{2}{m+2}.
\end{equation}

Now average over $S_0$.  An atom $x$ of mass $p_x$ lies in $U_0$ exactly when it is not observed in $S_0$, so
\[
        \bbE_{S_0}\sum_{x\in\At(P_X)\cap U_0}p_x^2
        =
        \sum_{x\in\At(P_X)}p_x^2(1-p_x)^{n_0}.
\]
For $p\in[0,1]$, $p(1-p)^{n_0}\le c_{\rm atom}/(n_0+1)$ for a universal constant $c_{\rm atom}$; hence
\[
        p_x^2(1-p_x)^{n_0}
        \le
        \frac{c_{\rm atom}}{n_0+1}p_x .
\]
Summing over atoms and using $\sum_xp_x\le1$ gives
\[
        \bbE_{S_0,S_2}
        \sum_{C\in\Pi_{\rm unobs}}P_X(C)^2
        \le
        \frac{c_{\rm atom}}{n_0}+\frac{2}{m+2}
        \le
        c_{\rm part}\frac{\alpha^2}{L}
\]
for a universal constant $c_{\rm part}$.  Choosing $C_2$ large enough and applying Markov's inequality gives the event $\mathcal E_{\rm part}$ with probability at least $0.95$.
\end{proof}

\begin{remark}[The third split is optional]\label{rem:two-split-partition}
We use a separate partition sample $S_2$ to keep the proof bookkeeping clean: $S_0$ is used to learn the interval hints, $S_1$ is used to learn the randomized predictor, and $S_2$ defines the rounding cells. But the third sample $S_2$ is not necessary.   One can instead sort the distinct contexts appearing in the confidence sample $S_0$ and use them as the lexicographic cutpoints for the unobserved region.  The same exchangeability argument applies.  Thus the third split can be removed at the cost of coupling the interval and partition bookkeeping.
\end{remark}

\section{Rounding the randomized predictor}\label{sec:rounding}

The lemma below rounds the randomized predictor constructed in Section~\ref{sec:intervals}.  The rounding is by shared inverse-CDF seeds: each cell $C\in\Pi$ receives one independent uniform seed $U_C$, and every context $x\in C$ is assigned the grid value obtained by applying the inverse-CDF sampler for $Q_x$ to that same seed.  One independent sampler seed per rounding cell preserves all signed calibration tests when the observed-atom radius term and the unobserved cell-mass term are small.
The previous two sections establish exactly these bounds: Proposition~\ref{prop:learned-intervals} controls the mass-squared weighted interval radii on observed atoms, and Lemma~\ref{lem:partition} controls the summed squared masses of unobserved cells.

A small optimization improves the run-time using the structure of the online to batch predictors we produce. If $Q_x=T^{-1}\sum_{t=1}^Tq_t(x)$, then drawing from $Q_x$ is the same as first drawing $\tau\sim{\rm Unif}(\{1,\ldots,T\})$ and then drawing from $q_\tau(x)$.  Thus a cell sampler seed can be stored as a pair $(\tau_C,U_C)$; on query $x\in C$, the evaluator reconstructs only $q_{\tau_C}(x)$ and samples from it using $U_C$.

\begin{proposition}[One-seed rounding for finite test families]\label{prop:finite-test-rounding}
Fix the interval-hint system $(I_x,\Lambda_x)_{x\in\calX}$ constructed from $S_0$ and the rounding partition $\Pi$ constructed from $S_0$ and $S_2$, and suppose $S_0$ lies in the coverage event $\mathcal E_{\rm cov}$ from Proposition~\ref{prop:learned-intervals}.  Let $Q$ be any randomized grid predictor satisfying $\supp(Q_x)\subseteq\Lambda_x$ for every $x$.  Let $\calA\subseteq[-1,1]^{\calX\times[0,1]}$ be a finite nonempty family of tests, and define
\[
        V_\Pi
        =
        \sum_{x\in O_0\cap\At(P_X)}p_x^2r_x^2
        +
        \sum_{C\in\Pi_{\rm unobs}}P_X(C)^2 .
\]
For each rounding cell $C\in\Pi$, draw an independent seed $U_C\sim\operatorname{Unif}[0,1]$.  Enumerate $\Lambda=\{v_1,\ldots,v_K\}$ and set
\[
        F_x(u)=v_j
        \quad\text{when}\quad
        \sum_{i<j}Q_x(v_i)<u\le \sum_{i\le j}Q_x(v_i),
\]
with the endpoint convention $F_x(0)=v_1$.  For $x\in C\in\Pi$, define $h(x)=F_x(U_C)$.  Then for every $\delta\in(0,1)$, with probability at least $1-\delta$ over the cell seeds,
\[
\max_{a\in\calA}
\left|
\bbE\bigl[a(X,h(X))(h(X)-Y)\bigr]
-
\bbE\left[\sum_{v\in\Lambda}Q_X(v)a(X,v)(v-Y)\right]
\right|
\le
C_{\rm rnd}
\left(
        \sqrt{V_\Pi\log\frac{2|\calA|}{\delta}}
        +\gamma
\right)
\]
for a universal constant $C_{\rm rnd}$.
\end{proposition}

\begin{proof}
Write $I_x=[a_x,b_x]$ and define
\[
        \clip_x(v)=\min\{\max\{v,a_x\},b_x\}
\]
for $v\in\Lambda_x$.  By the definition of $\Lambda_x$, $|v-\clip_x(v)|\le\gamma$ whenever $v\in\Lambda_x$.  If $x\notin O_0$, then $I_x=[0,1]$, $\Lambda_x=\Lambda$, and $\clip_x(v)=v$.  Values of $\clip_x(v)$ for $v\notin\Lambda_x$ are irrelevant because $Q_x(v)=0$.

Fix a test $a\in\calA$.  Since both predictors are functions of $X$, replacing $Y$ by $\mu(X)$ does not change either expectation.  Split $v-\mu(x)$ into the clipped part and the grid-rounding part:
\[
        v-\mu(x)
        =
        \bigl(\clip_x(v)-\mu(x)\bigr)
        +
        \bigl(v-\clip_x(v)\bigr).
\]
For a rounding cell $C$, let
\[
        Z^0_{a,C}
        =
        \int_C
        \left(
        a(x,F_x(U_C))(\clip_x(F_x(U_C))-\mu(x))
        -
        \sum_{v\in\Lambda}Q_x(v)a(x,v)(\clip_x(v)-\mu(x))
        \right)\,dP_X(x).
\]
For fixed $x$, $F_x(U_C)$ has distribution $Q_x$, so $\bbE[Z^0_{a,C}]=0$.  The variables $\{Z^0_{a,C}:C\in\Pi\}$ are independent because the cell seeds are independent.

We bound their ranges cell by cell.  If $C=\{x\}$ is an observed singleton with $p_x>0$, then $x$ is an atom.  On the coverage event, both $\clip_x(v)$ and $\mu(x)$ lie in $I_x$ for every $v\in\supp(Q_x)$, so
\[
        |a(x,v)(\clip_x(v)-\mu(x))|\le 2r_x .
\]
Thus $Z^0_{a,\{x\}}$ has range length at most $4p_xr_x$; if $p_x=0$, the range length is zero.  If $C\in\Pi_{\rm unobs}$, then $I_x=[0,1]$ on $C$ and $\clip_x(v)=v$, so the clipped payoff is bounded by one and $Z^0_{a,C}$ has range length at most $2P_X(C)$.

Hoeffding's inequality therefore gives, for $Z^0_a=\sum_C Z^0_{a,C}$,
\[
        \bbP\left(|Z^0_a|>t\right)
        \le
        2\exp\left(-\frac{c t^2}{V_\Pi}\right)
\]
with the interpretation that the probability is zero when $V_\Pi=0$ and $t>0$.  Taking
\[
        t=C\sqrt{V_\Pi\log\frac{2|\calA|}{\delta}}
\]
and union bounding over $a\in\calA$ controls the clipped part simultaneously for all tests.

The grid-rounding part is deterministic.  For every $a\in\calA$,
\[
\left|
\int
\left(
a(x,h(x))(h(x)-\clip_x(h(x)))
-
\sum_{v\in\Lambda}Q_x(v)a(x,v)(v-\clip_x(v))
\right)\,dP_X(x)
\right|
\le 2\gamma,
\]
because both $h(x)$ and every $v$ with $Q_x(v)>0$ lie in $\Lambda_x$.  Combining the clipped and grid-rounding bounds and absorbing constants proves the proposition.  The proof used only that the cell seeds are independent, that for each fixed $x$ the rounded value has marginal distribution $Q_x$, and that rounded values lie in $\Lambda_x$.  Therefore the same guarantee applies to the equivalent mixture sampler described above.
\end{proof}

\begin{lemma}[One seed per rounding cell preserves calibration]\label{lem:one-seed-rounding}
Fix $S_0$ lying in the events $\mathcal E_{\rm cov}\cap\mathcal E_{\rm rad}$ from Proposition~\ref{prop:learned-intervals}.  Fix $S_2$ lying in the partition event $\mathcal E_{\rm part}$ of Lemma~\ref{lem:partition}.  Let $Q$ be any randomized grid predictor such that $\supp(Q_x)\subseteq\Lambda_x$ for every $x$, and assume that the allowed-grid radius satisfies $\gamma\le\alpha/2$.  For each rounding cell $C\in\Pi$, draw an independent seed $U_C\sim\operatorname{Unif}[0,1]$.    Enumerate $\Lambda=\{v_1,\ldots,v_K\}$ and define the inverse-CDF sampler
\[
        F_x(u)=v_j
        \quad\text{when}\quad
        \sum_{i<j}Q_x(v_i)<u\le \sum_{i\le j}Q_x(v_i),
\]
with the endpoint convention $F_x(0)=v_1$.  For $x\in C\in\Pi$, set
\[
        h(x)=F_x(U_C).
\]
Equivalently, if $Q_x=T^{-1}\sum_tq_t(x)$, one may draw for each cell an independent pair $(\tau_C,U_C)$ with $\tau_C$ uniform on $\{1,\ldots,T\}$ and $U_C$ uniform on $[0,1]$, and set $h(x)$ by applying the inverse-CDF sampler for $q_{\tau_C}(x)$ to $U_C$.
Then, with probability at least $0.99$ over the cell seeds, the deterministic grid predictor $h:\calX\to\Lambda$ satisfies
\[
        \operatorname{MC}_P(h;\calG)
        \le
        \operatorname{MC}_P(Q;\calG)
        +
        C_3\alpha .
\]
Consequently, there exists a deterministic choice of the cell seeds with the same guarantee.
\end{lemma}

\begin{proof}
Apply Proposition~\ref{prop:finite-test-rounding} to the finite test family
\[
        \calA_{\rm MC}
        =
        \{(x,v)\mapsto g(x)\sigma(v):
        g\in\calG,\ \sigma\in\{\pm1\}^{\Lambda}\}.
\]
Its size is $M=|\calG|2^K$, and by the definition of $L$ we have $\log M\le L$.  On the radius and partition events,
\[
        V_\Pi
        =
        \sum_{x\in O_0\cap\At(P_X)}p_x^2r_x^2
        +
        \sum_{C\in\Pi_{\rm unobs}}P_X(C)^2
        \le
        C\frac{\alpha^2}{L}.
\]
Indeed, the observed-atom sum is dominated by the full atomic sum in Proposition~\ref{prop:learned-intervals}, while the unobserved-cell sum is controlled by Lemma~\ref{lem:partition}.  Since $L\ge100$, Proposition~\ref{prop:finite-test-rounding} with $\delta=0.01$ gives, with probability at least $0.99$ over the cell seeds, a deterministic predictor $h$ satisfying, for all $g\in\calG$ and $\sigma\in\{\pm1\}^{\Lambda}$,
\[
\left|
\bbE\left[g(X)\sigma(h(X))(h(X)-Y)\right]
-
\bbE\left[g(X)\sum_{v\in\Lambda}\sigma(v)Q_X(v)(v-Y)\right]
\right|
\le C'\alpha,
\]
where we used $\gamma\le\alpha/2$ and $\log(200M)\le C L$.

Finally, apply Lemma~\ref{lem:signed-form}.  For any $g$,
\[
\begin{aligned}
\sum_{v\in\Lambda}|B_P(h;g,v)|
&=
\max_{\sigma}
\bbE[g(X)\sigma(h(X))(h(X)-Y)]\\
&\le
\max_{\sigma}
\bbE\left[g(X)\sum_v\sigma(v)Q_X(v)(v-Y)\right]
+C'\alpha\\
&\le
\operatorname{MC}_P(Q;\calG)+C'\alpha.
\end{aligned}
\]
Taking the maximum over $g$ proves the lemma after increasing $C_3$ if necessary.
\end{proof}

\section{Putting the deterministic learner together}

We now give the full learning procedure.  Its output predictor is deterministic, although the training procedure uses randomness to choose the final cell sampler seeds.  Appendix~\ref{app:training-derandomization} shows how to remove these final training coins by enumerating limited-independent seed choices and selecting one on a validation sample.

At a high level, the learner uses the confidence sample $S_0$ to build intervals, the online-learning sample $S_1$ to learn a randomized predictor constrained to those intervals, and the partition sample $S_2$ to create finitely many rounding cells.  It then chooses one sampler seed per rounding cell and uses those seeds to turn the randomized predictor into a deterministic function.

\begin{center}
\begingroup
\setlength{\fboxsep}{0.75em}
\fbox{%
\begin{minipage}{0.94\textwidth}
\begin{algorithm}{Learn, average, and round with one sampler seed per rounding cell}\label{alg:main}
    \item \textbf{Input:} accuracy $\eps\in(0,1/10)$, group family $\calG$, and an i.i.d. sample $S$ from $P$.
    \item Set the internal accuracy $\alpha=c\eps$ for a sufficiently small universal constant $c>0$.  Set $\gamma=\alpha/64$, build the grid $\Lambda$, define $K,L,J,n_0$ as in Section~\ref{sec:intervals}, and set
    \[
            T=C_0\frac{L}{\alpha^2},
            \qquad
            m=C_0\frac{L}{\alpha^2}.
    \]
    \item Split $S$ into three independent parts: the confidence sample $S_0$, the online-learning sample $S_1$, and the partition sample $S_2$, with
    \[
            |S_0|=n_0=C_0\frac{LJ}{\alpha^2},
            \qquad
            |S_1|=T,
            \qquad
            |S_2|=m.
    \]
    \item Using the confidence sample $S_0$, compute the counts $N_x$, empirical means $\widehat\mu_x$, confidence intervals $I_x$, and allowed grid sets
    \[
        \Lambda_x=\{v\in\Lambda:\dist(v,I_x)\le\gamma\}.
    \]
    These define the learned interval-hint system.
    \item Run Algorithm~\ref{alg:known-interval-ew}, or its factored implementation Algorithm~\ref{alg:known-interval-factored}, on the online-learning sample $S_1$ with the learned interval-hint system $(I_x,\Lambda_x)_{x\in\calX}$ and group family $\calG$.  Let $Q$ be the resulting randomized grid predictor.
    \item Sort the distinct contexts in the partition sample $S_2$ lexicographically and form the finite rounding-cell family $\Pi$ from Section~\ref{sec:partition}.
    \item Round $Q$ to a deterministic predictor $h$ as in Lemma~\ref{lem:one-seed-rounding}.  For the query-efficient implementation, draw for each rounding cell $C\in\Pi$ an independent pair $(\tau_C,U_C)$ with $\tau_C$ uniform on $\{1,\ldots,T\}$ and $U_C$ uniform on $[0,1]$; for $x\in C$, reconstruct $q_{\tau_C}(x)$ and set $h(x)$ by applying its inverse-CDF sampler to $U_C$.
    \item \textbf{Output:} the deterministic grid predictor $h:\calX\to\Lambda$.
\end{algorithm}
\end{minipage}}
\endgroup
\end{center}

\begin{theorem}[Deterministic predictors with randomized-rate sample complexity]\label{thm:main}
There is a universal constant $C$ such that the following holds.  For every distribution $P$ on $\calX\times[0,1]$, every finite nonempty group family $\calG\subseteq[0,1]^{\calX}$, and every $\eps\in(0,1/10)$, Algorithm~\ref{alg:main}, with a sufficiently small universal choice $\alpha=c\eps$, uses
\[
        n
        \le
        C\frac{(1/\eps+\log(|\calG|+1))\log((1/\eps+\log(|\calG|+2))/\eps)}{\eps^2}
\]
samples and outputs a deterministic predictor $h:\calX\to\Lambda$ satisfying
\[
        \operatorname{MC}_P(h;\calG)
        \le
        \eps
\]
with probability at least $2/3$ over the sample and the training randomness.
In particular, if $|\calG|\le\eps^{-\kappa}$ for any fixed $\kappa>0$, then
\[
        n=\widetilde O(\eps^{-3}).
\]

Moreover, Algorithm~\ref{alg:main} has an implicit representation with the same statistical guarantee, training time polynomial in $n_0$, $T$, $m$, $|\calG|$, $K$, and $d$, and per-query evaluation time polynomial in $|\calG|$, $K$, $d$, $\log T$, and $\log m$.  In the polynomial group regime $|\calG|\le\eps^{-\kappa}$, this running time is polynomial in $1/\eps$ and $d$.
\end{theorem}

The polynomial implementation claim follows from the factored exponential-weights representation and finite query reconstruction in Appendix~\ref{app:efficient}.

\begin{proof}
Let $\mathcal E_{\rm cov}$ and $\mathcal E_{\rm rad}$ be the events from Proposition~\ref{prop:learned-intervals}, and set $\mathcal E_0=\mathcal E_{\rm cov}\cap\mathcal E_{\rm rad}$.  By that proposition, after adjusting constants,
\[
        \bbP(\mathcal E_0)\ge0.85.
\]
Let $\mathcal E_{\rm part}$ be the partition event of Lemma~\ref{lem:partition}.  Since $\bbP(\mathcal E_{\rm part})\ge0.95$, a union bound gives
\[
        \bbP(\mathcal E_0\cap\mathcal E_{\rm part})\ge0.80 .
\]
Condition on the confidence and partition samples $S_0,S_2$ such that $\mathcal E_0\cap\mathcal E_{\rm part}$ holds.

Set $U_x=I_x$.  Proposition~\ref{prop:learned-intervals} gives $\mu(x)\in U_x$ for every $x$, and $\Lambda_x$ $\gamma$-covers $U_x$ for every $x$.  Thus, on the event we have conditioned on, the learned interval-hint system is $\gamma$-valid.  Apply Theorem~\ref{thm:constrained-online} with failure probability $0.05$.  Since
\[
        1/\gamma+\log|\calG|+O(1)\le C L
\]
and
\[
        T=C_0\frac{L}{\alpha^2},
\]
taking $C_0$ sufficiently large gives, with conditional probability at least $0.95$ over the online-learning sample $S_1$ and the online procedure,
\begin{equation}\label{eq:Q-good}
        \operatorname{MC}_P(Q;\calG)
        \le
        C_{\rm on}\left(\gamma+\sqrt{\frac{L}{T}}\right)
        \le
        C_4\alpha.
\end{equation}
The predictor $Q$ is supported on $\Lambda_x$ for every $x$.

The algorithm sets $\gamma=\alpha/64$, so the hypothesis $\gamma\le\alpha/2$ in Lemma~\ref{lem:one-seed-rounding} holds.  On the event \eqref{eq:Q-good}, that lemma gives, with probability at least $0.99$ over the rounding seeds, a deterministic predictor $h$ with
\[
        \operatorname{MC}_P(h;\calG)
        \le
        \operatorname{MC}_P(Q;\calG)+C_3\alpha
        \le
        (C_3+C_4)\alpha.
\]
Choose the universal constant $c$ small enough that $(C_3+C_4)\alpha\le\eps$.

The final rounding succeeds with probability at least $0.99$ by Lemma~\ref{lem:one-seed-rounding}.  Therefore the total success probability is at least
\[
        0.80\cdot0.95\cdot0.99>2/3.
\]

Finally,
\[
        K=O(1/\alpha)=O(1/\eps),
        \qquad
        L=O(1/\eps+\log(|\calG|+1)),
\]
and
\[
        J=O\left(\log\left(\frac{1/\eps+\log(|\calG|+2)}{\eps}\right)\right).
\]
Thus
\[
        n=n_0+T+m
        =O\left(\frac{LJ}{\alpha^2}\right)
        \le
        C\frac{(1/\eps+\log(|\calG|+1))\log((1/\eps+\log(|\calG|+2))/\eps)}{\eps^2}.
\]
If $|\calG|\le\eps^{-\kappa}$ for fixed $\kappa$, then $L=\widetilde O(1/\eps)$, and hence $n=\widetilde O(\eps^{-3})$.

It remains to justify the implementation and query-time claims in the theorem.  During training, the confidence-interval data from $S_0$ is computed by scanning the $n_0$ confidence-sample points and maintaining the observed-context table.  The partition data from $S_2$ is computed by sorting the $m$ partition-sample contexts lexicographically and assigning one sampler seed $(\tau_C,U_C)$ to each rounding cell induced by the sorted cutpoints.  Appendix~\ref{app:efficient} shows that each exponential-weights state is represented by the $|\calG|K$ numbers $S_{g,v}^{t-1}$ rather than by $|\calG|2^K$ weights over signed calibration tests, that each online round solves one linear program with at most $K+1$ variables, and that each update costs polynomial time in $|\calG|K$.  Thus training is polynomial in $n_0$, $T$, $m$, $|\calG|$, $K$, and $d$.

For query evaluation, the implicit evaluator locates the query's rounding cell using $O(d\log m)$ lexicographic comparisons, retrieves its stored index $\tau_C$, and reconstructs only $q_{\tau_C}(x)$ from the corresponding online state.  It then applies the inverse-CDF sampler for $q_{\tau_C}(x)$ using the stored uniform seed $U_C$.  Therefore each query is polynomial in $|\calG|$, $K$, $d$, $\log T$, and $\log m$, with no need to reconstruct all $T$ online predictions.  In the polynomial group regime, the bounds on $n_0$, $T$, $m$, and $K$ above are polynomial in $1/\eps$, including the logarithmic factors.
\end{proof}

We stated our main theorem for finite classes, but it easily extends to infinite classes that have finite covers.

\begin{corollary}[Infinite group classes from covers]\label{cor:mc-covers}
Let $\calG\subseteq[0,1]^\calX$ be any nonempty group class, and interpret $\operatorname{MC}_P(h;\calG)$ with the supremum over $g\in\calG$.  For a marginal distribution $P_X$, let
$N_1(\rho,\calG,P_X)$ 
denote the least cardinality of an $L_1(P_X)$ $\rho$-cover of $\calG$.  There are universal constants $C,c>0$ such that the following holds.  Fix any distribution $P$ on $\calX\times[0,1]$ with marginal $P_X$, and let $\eps\in(0,1/10)$.  If $N=N_1(c\eps,\calG,P_X)<\infty$ and such a cover is supplied to the learner, then there is an algorithm using
\[
        n
        \le
        C\frac{(1/\eps+\log(N+1))\log((1/\eps+\log(N+2))/\eps)}{\eps^2}
\]
samples that outputs a deterministic predictor $h:\calX\to\Lambda$ satisfying
\[
        \operatorname{MC}_P(h;\calG)\le\eps
\]
with probability at least $2/3$ over the sample and the training randomness.
\end{corollary}

\begin{proof}
Let $\calG_0$ be the supplied $L_1(P_X)$ cover at radius $c\eps$, and run Theorem~\ref{thm:main} on the finite family $\calG_0$ with target accuracy $\eps/2$.  With probability at least $2/3$, the resulting deterministic predictor $h$ satisfies $\operatorname{MC}_P(h;\calG_0)\le\eps/2$.  Now fix any $g\in\calG$ and choose $g_0\in\calG_0$ with $\bbE|g(X)-g_0(X)|\le c\eps$.  By the signed form in Lemma~\ref{lem:signed-form}, for every sign pattern $\sigma\in\{\pm1\}^{\Lambda}$,
\[
 \left|
 \bbE\left[(g-g_0)(X)\sigma(h(X))(h(X)-Y)\right]
 \right|
 \le
 \bbE|g(X)-g_0(X)|
 \le c\eps,
\]
since $h,Y\in[0,1]$.  Thus the signed calibration error for $g$ is at most the signed calibration error for $g_0$ plus $c\eps$.  Taking the supremum over $g$ and choosing $c$ small enough gives $\operatorname{MC}_P(h;\calG)\le\eps$.  The sample bound is the bound in Theorem~\ref{thm:main}, with constants adjusted for the target accuracy $\eps/2$.
\end{proof}

\begin{remark}[Pseudo-dimension classes]\label{rem:mc-pdim}
If $\calG\subseteq[0,1]^\calX$ has pseudo-dimension $p$, standard covering estimates for bounded real-valued classes~\citep{haussler1992decision} imply that, for every marginal $P_X$ and every $\rho>0$,
\[
        N_1(\rho,\calG,P_X)
        \le
        \left(\frac{C}{\rho}\right)^{Cp}
\]
for a universal constant $C$.  Thus, when such a cover is available, Corollary~\ref{cor:mc-covers} gives deterministic ECE multicalibration with sample complexity
\[
        \widetilde O\left(\eps^{-3}+\frac{p}{\eps^2}\right).
\]
If the marginal distribution is not known, one can build the cover from an additional context-only sample of size $\widetilde O(p/\eps^2)$ by the usual empirical-cover argument, in the same way as in Corollary~\ref{cor:pdim-det-omni}.  This keeps the total sample complexity at $\widetilde O(\eps^{-3}+p/\eps^2)$.
\end{remark}

\section{Outcome indistinguishability and omniprediction}\label{sec:linear-tests}

The scalar theorem was stated for ECE multicalibration, but was specialized to multicalibration only through the finite family of signed OI tests 
$(x,v)\mapsto g(x)\sigma(v).$ 
In this section we isolate the more general principle.  Suppose we are given any finite collection of bounded OI tests
$\calA\subseteq[-1,1]^{\calX\times[0,1]}.$ 
The goal is to find a deterministic predictor whose residual $h(X)-Y$ has small correlation with every test evaluated at the context and the prediction.  The theorem below says that the same interval-hint learning idea and one-seed rounding construction give a deterministic predictor satisfying these constraints at the nearly optimal rate.

\subsection{Finite OI test families}

Fix a finite grid $\Lambda\subseteq[0,1]$.  For a randomized grid predictor $Q:\calX\to\Delta(\Lambda)$ and a finite test family $\calA$, define the OI error as:
\[
        \operatorname{OIErr}_P(Q;\calA)
        =
        \max_{a\in\calA}
        \left|
        \bbE\left[
        \sum_{v\in\Lambda}Q_X(v)a(X,v)(v-Y)
        \right]
        \right|.
\]
For a deterministic predictor $h:\calX\to\Lambda$, this specializes to
\[
        \operatorname{OIErr}_P(h;\calA)
        =
        \max_{a\in\calA}
        \left|
        \bbE\left[
        a(X,h(X))(h(X)-Y)
        \right]
        \right|.
\]

The lemma below is the finite-test analogue of Theorem~\ref{thm:constrained-online}.  Its proof uses the same interval-hint online-to-batch mechanism as Section~\ref{sec:intervalhints}: taking
\[
        \calA=\{(x,v)\mapsto g(x)\sigma(v):g\in\calG,\ \sigma\in\{\pm1\}^{\Lambda}\}
\]
recovers the signed-test guarantee underlying Theorem~\ref{thm:constrained-online}, with Lemma~\ref{lem:signed-form} converting that guarantee to ECE multicalibration.

Concretely, the learner in the lemma is the finite-test variant of Algorithm~\ref{alg:known-interval-ew}, with the signed calibration class replaced by
\[
        \calA_\pm=\{a,-a:a\in\calA\}.
\]
At round $t$ it maintains an exponential-weights distribution $\pi_t$ over $\calA_\pm$, forms the mixture
\[
        c_t(x,v)=\sum_{a\in\calA_\pm}\pi_t(a)a(x,v),
\]
predicts with $q_t(x)=\operatorname{LP}_x(c_t(x,\cdot))$, updates the weights using the realized residuals
\[
        z_t^a=\sum_{v\in\Lambda}q_t(X_t)(v)a(X_t,v)(v-Y_t),
\]
and outputs the averaged predictor $Q_x=T^{-1}\sum_tq_t(x)$.

\begin{lemma}[Finite-test learning from valid interval hints]\label{lem:finite-test-otb}
There is a universal constant $C_{\rm oi}$ such that the following holds.  Let $\calA\subseteq[-1,1]^{\calX\times[0,1]}$ be a finite nonempty family of tests, let $\Lambda\subseteq[0,1]$ be a finite grid, let $\delta\in(0,1)$, and suppose the learner is given a $\gamma$-valid interval-hint system $(U_x,\Lambda_x)_{x\in\calX}$ for $P$.  Then there is a learning algorithm which, given $T$ i.i.d.\ samples from $P$ and the interval-hint system, outputs a randomized grid predictor $Q$ satisfying
\[
        \supp(Q_x)\subseteq\Lambda_x\qquad\text{for every }x,
\]
and, with probability at least $1-\delta$,
\[
        \operatorname{OIErr}_P(Q;\calA)
        \le
        C_{\rm oi}\left(
        \gamma+
        \sqrt{\frac{\log(2|\calA|)+\log(1/\delta)}{T}}
        \right).
\]
\end{lemma}

\begin{proof}
This is the same interval-hint online-to-batch proof as Lemmas~\ref{lem:hinted-online} and~\ref{lem:martingale-otb}, with the finite test family $\calA$ and its negatives replacing the signed calibration tests.  The endpoint LP is run against the current exponential-weights mixture of tests, the valid interval hint gives a one-step $\gamma$ minimax bound, and the martingale online-to-batch step turns the online residual guarantees into population OI guarantees for the averaged predictor.  Appendix~\ref{app:finite-oi-proof} gives the full argument.
\end{proof}

\begin{theorem}[Deterministic outcome indistinguishability for finite test families]\label{thm:linear-tests}
There is a universal constant $C$ such that the following holds.  Let $\calA\subseteq[-1,1]^{\calX\times[0,1]}$ be a finite nonempty test family.  For every distribution $P$ on $\calX\times[0,1]$ and every $\eps\in(0,1/10)$, there is an algorithm using
\[
        n
        \le
        C\frac{L_{\calA}\log(100L_{\calA}/\eps)}{\eps^2},
        \qquad
        L_{\calA}=\log(2|\calA|+2)+100,
\]
samples that outputs a deterministic grid predictor $h:\calX\to\Lambda$ satisfying
\[
        \operatorname{OIErr}_P(h;\calA)\le\eps
\]
with probability at least $2/3$ over the sample and the training randomness.  In particular, if $|\calA|\le\eps^{-\kappa}$ for any fixed $\kappa>0$, then $n=\widetilde O(\eps^{-2})$.
\end{theorem}

\begin{proof}
The algorithm is the same three-way sample split as Algorithm~\ref{alg:main}.  Set an internal accuracy $\alpha=c\eps$ for a sufficiently small universal constant $c$, set $\gamma=\alpha/64$, and take the same uniform $\gamma$-net grid $\Lambda$.  The only change is that the quantity called $L$ in Sections~\ref{sec:intervals}--\ref{sec:rounding} is replaced by
\[
        L=L_{\calA}=\log(2|\calA|+2)+100,
\]
and $J=100\log(100L/\alpha)$.  Use
\[
        n_0=C_0\frac{LJ}{\alpha^2},
        \qquad
        T=C_0\frac{L}{\alpha^2},
        \qquad
        m=C_0\frac{L}{\alpha^2}.
\]
The confidence sample $S_0$ builds intervals $I_x$ and allowed grid sets $\Lambda_x$ exactly as in Section~\ref{sec:intervals}; the partition sample $S_2$ builds the rounding cells exactly as in Section~\ref{sec:partition}.  Proposition~\ref{prop:learned-intervals} and Lemma~\ref{lem:partition} are unchanged with this value of $L$: with probability at least $0.80$, the learned interval hints are valid, the atom radius bound is at most $C\alpha^2/L$, and the unobserved cell-mass sum is at most $C\alpha^2/L$.

Condition on these events.  The learned interval-hint system is $\gamma$-valid, so Lemma~\ref{lem:finite-test-otb} with failure probability $0.05$ produces a randomized grid predictor $Q$ supported on $\Lambda_x$ such that
\[
        \operatorname{OIErr}_P(Q;\calA)
        \le
        C\left(\gamma+\sqrt{\frac{L}{T}}\right)
        \le C'\alpha .
\]

Finally round $Q$ using one independent seed per rounding cell, as in Proposition~\ref{prop:finite-test-rounding}.  On the same radius and partition events,
\[
        V_\Pi
        =
        \sum_{x\in O_0\cap\At(P_X)}p_x^2r_x^2
        +
        \sum_{C\in\Pi_{\rm unobs}}P_X(C)^2
        \le
        C\left(
        \frac{\alpha^2}{L}
        \right)
\]
and $\log(200|\calA|)\le C L$.  Proposition~\ref{prop:finite-test-rounding} with $\delta=0.01$ therefore gives, with probability at least $0.99$ over the cell seeds, a deterministic grid predictor $h$ satisfying
\[
        \operatorname{OIErr}_P(h;\calA)
        \le
        \operatorname{OIErr}_P(Q;\calA)+C\alpha
        \le C''\alpha .
\]
Choosing $c$ small enough gives the target error $\eps$, and the success probabilities combine exactly as in Theorem~\ref{thm:main}.  The stated sample bound follows from the definitions of $n_0,T,m,L$, and $J$.
\end{proof}

\subsection{Omniprediction}

We now specialize Theorem~\ref{thm:linear-tests} to omniprediction.  In this subsection outcomes are binary, $Y\in\{0,1\}$.  For a loss $\ell:[0,1]\times\{0,1\}\to[0,1]$ and a value $p\in[0,1]$, call any minimizer of
\[
        (1-p)\ell(a,0)+p\ell(a,1)
\]
over $a\in[0,1]$ a \emph{Bayes act} for $\ell$ at $p$.  We assume throughout that, for each loss considered below, we have fixed one such Bayes act at every $p$, and write this choice as
\[
        k_\ell(p)\in
        \operatorname*{argmin}_{a\in[0,1]}
        \bbE_{Z\sim{\rm Bernoulli}(p)}[\ell(a,Z)] ,
\]
using the chosen act when there are multiple minimizers.  We also write
$\Delta\ell(u)=\ell(u,1)-\ell(u,0).$

\begin{definition}[\cite{gopalan2021omnipredictors}]
A deterministic predictor $h:\calX\to[0,1]$ is an $(\calL,\calH,\eps)$-omnipredictor if
\[
        \bbE[\ell(k_\ell(h(X)),Y)]
        \le
        \inf_{f\in\calH}\bbE[\ell(f(X),Y)]+\eps
        \qquad\text{for every }\ell\in\calL .
\]
\end{definition}
Given a loss class $\calL$ and a hypothesis class $\calH\subseteq[0,1]^\calX$, let
\[
        \Delta\calL\circ\calH
        =
        \{x\mapsto \Delta\ell(f(x)):\ell\in\calL,\ f\in\calH\}.
\]
We assume below that all functions in the auditor classes are bounded in $[-1,1]$; other bounded ranges can be handled by rescaling.

The route from OI tests to omniprediction follows the Loss OI reduction of \citet{gopalan2023loss}, in the threshold-calibration and multiaccuracy formulation of \citet[Section~3]{okoroafor2025near}.  It has three ingredients.  Threshold tests in the prediction value give threshold calibration; context-only auditor tests give multiaccuracy for the loss-derived class $\Delta\calL\circ\calH$; and the  omniprediction argument of \cite{gopalan2023loss,okoroafor2025near} combines threshold calibration and multiaccuracy.  We spell out these implications in our notation below.

For $\theta\in[0,1]$, let
\[
        T_\theta(v)=\one\{v\le\theta\}.
\]
When predictions are supported on a finite grid $\Lambda$, only $|\Lambda|+1$ different threshold functions can appear on $\Lambda$.  Let $\Theta_\Lambda$ contain one representative for each distinct restriction of $T_\theta$ to $\Lambda$.

For a randomized grid predictor $Q:\calX\to\Delta(\Lambda)$, define threshold calibration and multiaccuracy errors by
\[
        \TCal_P(Q;\Theta_\Lambda)
        =
        \max_{\theta\in\Theta_\Lambda}
        \left|
        \bbE\left[
        \sum_{v\in\Lambda}Q_X(v)\one\{v\le\theta\}(v-Y)
        \right]
        \right|
\]
and
\[
        \MA_P(Q;\calC)
        =
        \max_{c\in\calC}
        \left|
        \bbE\left[
        c(X)\sum_{v\in\Lambda}Q_X(v)(v-Y)
        \right]
        \right|.
\]
For deterministic $h$, we again identify $h$ with the point-mass predictor $Q_X=\delta_{h(X)}$.

We next record the elementary monotonicity fact that converts the finite threshold tests into the errors used by the omniprediction reduction.  This is the population version of the Bayes-act/proper-loss observation used by \citet[Theorem~3.1]{okoroafor2025near}: after Bayes-act postprocessing, an arbitrary loss induces a threshold-calibration residual.

\begin{lemma}[Threshold calibration controls Bayes-act residuals]\label{lem:threshold-bayes-act}
Let $\ell:[0,1]\times\{0,1\}\to[0,1]$ be a bounded loss with a fixed Bayes-act choice $k_\ell$, and set
\[
        s_\ell(v)=\Delta\ell(k_\ell(v)).
\]
If $Q$ is supported on a finite grid $\Lambda$, $\Theta_\Lambda$ contains one representative for every threshold sublevel-set pattern on $\Lambda$, and $V\sim Q_X$ conditionally on $X$, then
\[
        \left|
        \bbE[(Y-V)s_\ell(V)]
        \right|
        \le
        3\,\TCal_P(Q;\Theta_\Lambda).
\]
\end{lemma}

\begin{proof}
We first show that $s_\ell$ is nonincreasing.  Fix $0\le a<b\le1$ and set
\[
        d_0=\ell(k_\ell(b),0)-\ell(k_\ell(a),0),
        \qquad
        d_1=\ell(k_\ell(b),1)-\ell(k_\ell(a),1).
\]
For $q\in[0,1]$, define the affine function
\[
        F(q)
        =
        (1-q)d_0+qd_1
        =
        \bbE_{Z\sim{\rm Bernoulli}(q)}
        [\ell(k_\ell(b),Z)-\ell(k_\ell(a),Z)] .
\]
Since $k_\ell(a)$ is a Bayes act at mean $a$, we have $F(a)\ge0$.  Since $k_\ell(b)$ is a Bayes act at mean $b$, we have $F(b)\le0$.  Because $F$ is affine and $a<b$, its slope is nonpositive:
\[
        d_1-d_0
        =
        \frac{F(b)-F(a)}{b-a}
        \le0 .
\]
But
\[
        d_1-d_0
        =
        \bigl(\ell(k_\ell(b),1)-\ell(k_\ell(b),0)\bigr)
        -
        \bigl(\ell(k_\ell(a),1)-\ell(k_\ell(a),0)\bigr)
        =
        s_\ell(b)-s_\ell(a).
\]
Thus $s_\ell$ is nonincreasing.  Also $s_\ell(v)\in[-1,1]$ because $\ell$ is $[0,1]$-valued.

Write the grid as $\Lambda=\{v_1<\cdots<v_K\}$.  Let $b_0=s_\ell(v_K)$ and, for $i<K$, let $b_i=s_\ell(v_i)-s_\ell(v_{i+1})\ge0$.  Then for every $v_j\in\Lambda$,
\[
        s_\ell(v_j)
        =
        b_0+\sum_{i=1}^{K-1}b_i\one\{v_j\le v_i\},
\]
and
\[
        |b_0|+\sum_{i=1}^{K-1}b_i\le3.
\]
The constant term is the threshold $\one\{v\le1\}$.  Therefore, by the definition of $\TCal_P(Q;\Theta_\Lambda)$, each term in this decomposition has residual correlation at most $\TCal_P(Q;\Theta_\Lambda)$ in absolute value.  Since
\[
        \bbE[(Y-V)s_\ell(V)]
        =
        -\bbE\left[
        \sum_{v\in\Lambda}Q_X(v)s_\ell(v)(v-Y)
        \right],
\]
the coefficient bound gives the claim.
\end{proof}

\begin{definition}[Finite auditor representation]\label{def:finite-auditor-representation}
Fix a marginal distribution $P_X$ and a class $\calD\subseteq[-1,1]^\calX$.  A finite auditor class $\calC\subseteq[-1,1]^\calX$ $(\lambda,\rho)$-represents $\calD$ under $P_X$ if every $c^\star\in\calD$ admits an approximation
\[
        \bbE\left[
        \left|
        c^\star(X)-\sum_i\beta_i c_i(X)
        \right|
        \right]\le\rho,
        \qquad
        c_i\in\calC,
        \qquad
        \sum_i|\beta_i|\le\lambda
\]
where the expectation is over $X\sim P_X$.  If the same inequality holds with $\sup_x$ in place of the expectation, then the representation is uniform and holds under every distribution.
\end{definition}

This representation packages the finite approximate bases and spanned classes used by \citet[Definitions~4.3--4.4]{okoroafor2025near} into the population finite-auditor form needed by our OI theorem.  Thus $\MA_P(Q;\calC)\le\beta$ implies multiaccuracy with respect to all of $\calD$ at error at most $\lambda\beta+\rho$, since the residual term is bounded by one.  This is the same linearity observation as in \citet[Lemmas~4.11 and~4.14]{okoroafor2025near}.

The next lemma is the Loss OI omniprediction reduction of \citet{gopalan2023loss}, adapted to the threshold-calibration formulation of \citet[Lemma~3.2]{okoroafor2025near}, randomized grid predictors, and the finite-auditor representation above.  We give the proof for completeness.

\begin{lemma}[Omniprediction from threshold calibration and multiaccuracy]\label{lem:omni-from-tests}
Let $P$ be a distribution on $\calX\times\{0,1\}$ with marginal $P_X$, let $Q$ be supported on a finite grid $\Lambda$, and let $\Theta_\Lambda$ contain one representative for every threshold sublevel-set pattern on $\Lambda$.  Suppose $\calC$ $(\lambda,\rho)$-represents $\Delta\calL\circ\calH$ under $P_X$.  If $Q$ satisfies
\[
        \TCal_P(Q;\Theta_\Lambda)\le\beta
        \qquad\text{and}\qquad
        \MA_P(Q;\calC)\le\beta,
\]
then $Q$ is an $(\calL,\calH,O((1+\lambda)\beta+\rho))$-omnipredictor, where for randomized $Q$ the loss is interpreted as
\[
        \bbE_{X,Y}\bbE_{V\sim Q_X}[\ell(k_\ell(V),Y)] .
\]
In particular, the same conclusion holds for deterministic predictors.
\end{lemma}

\begin{proof}
Let $V\sim Q_X$, conditionally independently given $X$, and let $Z\sim{\rm Bernoulli}(V)$, conditionally independently given $(X,V)$.  Fix a loss $\ell\in\calL$ and a benchmark $f\in\calH$.  By the definition of the Bayes act,
\[
        \bbE[\ell(k_\ell(V),Z)\mid X,V]
        \le
        \bbE[\ell(f(X),Z)\mid X,V],
\]
and hence the same inequality holds after averaging.

It remains to compare the predictor-induced label $Z$ with the true label $Y$ on the two sides.  For the postprocessed predictor,
\[
        \left|
        \bbE[\ell(k_\ell(V),Y)-\ell(k_\ell(V),Z)]
        \right|
        =
        \left|
        \bbE[(Y-V)\Delta\ell(k_\ell(V))]
        \right| .
\]
Lemma~\ref{lem:threshold-bayes-act} bounds this quantity by $3\TCal_P(Q;\Theta_\Lambda)=O(\beta)$.

For the benchmark side,
\[
        \left|
        \bbE[\ell(f(X),Y)-\ell(f(X),Z)]
        \right|
        =
        \left|
        \bbE[(Y-V)\Delta\ell(f(X))]
        \right| .
\]
The function $x\mapsto\Delta\ell(f(x))$ lies in $\Delta\calL\circ\calH$.  By the representation assumption, it is within $L_1(P_X)$ error $\rho$ of a linear combination $\sum_i\beta_i c_i$ with $c_i\in\calC$ and $\sum_i|\beta_i|\le\lambda$.  Since $|Y-V|\le1$, this gives
\[
        \left|
        \bbE[(Y-V)\Delta\ell(f(X))]
        \right|
        \le
        \lambda\,\MA_P(Q;\calC)+\rho
        \le
        \lambda\beta+\rho .
\]
Combining these two comparisons with the ideal inequality under $Z$, and then taking the infimum over $f\in\calH$ and the supremum over $\ell\in\calL$, proves the lemma.
\end{proof}

\begin{theorem}[Deterministic omniprediction from finite auditors]\label{thm:det-omni}
Let $\calL$ be a class of losses $\ell:[0,1]\times\{0,1\}\to[0,1]$ satisfying the Bayes-act convention above, let $\calH\subseteq[0,1]^\calX$, and let $\calC\subseteq[-1,1]^\calX$ be a finite auditor class.  There are universal constants $C,c>0$ such that the following holds.  For every $\eps\in(0,1/10)$ and every distribution $P$ on $\calX\times\{0,1\}$, with marginal $P_X$, if $\calC$ $(\lambda,\rho)$-represents $\Delta\calL\circ\calH$ under $P_X$ with $\rho\le c\eps$, then there is an algorithm using
\[
        n
        \le
        C\frac{(1+\lambda)^2 L_{\rm omni}
        \log(100(1+\lambda)L_{\rm omni}/\eps)}
        {\eps^2},
        \qquad
        L_{\rm omni}
        =
        \log\left(|\calC|+\frac{1+\lambda}{\eps}+2\right)+100,
\]
samples and outputs a deterministic grid predictor $h:\calX\to\Lambda$ that is an $(\calL,\calH,\eps)$-omnipredictor with probability at least $2/3$.  In particular, for constant $\lambda$ and $|\calC|\ge2$, this is
\[
        n=\widetilde O\left(\frac{\log|\calC|+\log(1/\eps)}{\eps^2}\right).
\]
\end{theorem}

\begin{proof}
Set the OI-test accuracy
\[
        \alpha=c_0\frac{\eps}{1+\lambda}
\]
for a sufficiently small universal constant $c_0$.  Use the construction from Theorem~\ref{thm:linear-tests} at accuracy $\alpha$: first build its uniform grid $\Lambda$, and then define the threshold representatives $\Theta_\Lambda$ for this grid.  Run the construction with the test family
\[
        \calA_{\rm omni}
        =
        \{(x,v)\mapsto c(x):c\in\calC\}
        \cup
        \{(x,v)\mapsto \one\{v\le\theta\}:\theta\in\Theta_\Lambda\}.
\]
Since $|\Lambda|=O(1/\alpha)=O((1+\lambda)/\eps)$, we have
\[
        |\calA_{\rm omni}|
        \le
        |\calC|+O\left(\frac{1+\lambda}{\eps}\right).
\]
The resulting deterministic predictor $h$ satisfies
\[
        \MA_P(h;\calC)\le\alpha,
\]
and
\[
        \TCal_P(h;\Theta_\Lambda)\le\alpha .
\]
The representation assumption gives multiaccuracy with respect to $\Delta\calL\circ\calH$ at error at most $\lambda\alpha+\rho$.  Lemma~\ref{lem:omni-from-tests} therefore implies that $h$ is an $(\calL,\calH,O((1+\lambda)\alpha+\rho))$-omnipredictor.  The choices of $c_0$ and $c$ make this error at most $\eps$.

The sample bound follows from Theorem~\ref{thm:linear-tests} with target accuracy $\alpha$ and the bound above on $|\calA_{\rm omni}|$.
\end{proof}

\begin{corollary}[Finite-class deterministic omniprediction]\label{cor:finite-det-omni}
If $\calL$ and $\calH$ are finite, then
\[
        n=\widetilde O\left(
        \frac{\log|\calL|+\log|\calH|+\log(1/\eps)}{\eps^2}
        \right)
\]
samples suffice for a deterministic $(\calL,\calH,\eps)$-omnipredictor, with constant success probability.
\end{corollary}

\begin{proof}
Take $\calC=\Delta\calL\circ\calH$, so $\lambda=1$, $\rho=0$, and $|\calC|\le|\calL||\calH|$.
\end{proof}

\begin{corollary}[Deterministic omniprediction from covers]\label{cor:cover-det-omni}
Let $\calD=\Delta\calL\circ\calH\subseteq[-1,1]^\calX$.  For a marginal distribution $P_X$, let
\[
        N_1(\rho,\calD,P_X)
\]
denote the least cardinality of an $L_1(P_X)$ $\rho$-cover of $\calD$.  There are universal constants $C,c>0$ such that the following holds.  Fix any distribution $P$ on $\calX\times\{0,1\}$ with marginal $P_X$, and let $\eps\in(0,1/10)$.  If $N=N_1(c\eps,\calD,P_X)<\infty$ and such a cover is supplied to the learner, then there is an algorithm using
\[
        n
        \le
        C\frac{L_{\rm cov}\log(100L_{\rm cov}/\eps)}{\eps^2},
        \qquad
        L_{\rm cov}
        =
        \log\left(N+\frac1\eps+2\right)+100,
\]
samples that outputs a deterministic $(\calL,\calH,\eps)$-omnipredictor with probability at least $2/3$.  The same conclusion holds distribution-free whenever $\calD$ has a uniform $c\eps$-cover of cardinality $N$.
\end{corollary}

\begin{proof}
Let $\calC$ be the supplied $L_1(P_X)$ cover.  Then $\calC$ $(1,c\eps)$-represents $\calD$ under $P_X$ in the sense of Definition~\ref{def:finite-auditor-representation}.  The result follows from Theorem~\ref{thm:det-omni}.
\end{proof}

\begin{corollary}[Pseudo-dimension auditor classes]\label{cor:pdim-det-omni}
Let $\calD=\Delta\calL\circ\calH\subseteq[-1,1]^\calX$ have pseudo-dimension $p$.  Then there is an algorithm that for every distribution $P$ on $\calX\times\{0,1\}$ and every $\eps\in(0,1/10)$ uses a fourth, context-only sample of size $\widetilde O(p/\eps^2)$ to construct a finite auditor cover and uses, in total,
\[
        n
        =
        \widetilde O\left(
        \frac{p+\log(1/\eps)}{\eps^2}
        \right)
\]
samples and outputs a deterministic $(\calL,\calH,\eps)$-omnipredictor with probability at least $2/3$.
\end{corollary}

\begin{proof}
Standard covering estimates for bounded real-valued classes of pseudo-dimension $p$~\citep{haussler1992decision} imply that for every marginal $P_X$ and every $\rho>0$,
\[
        N_1(\rho,\calD,P_X)
        \le
        \left(\frac{C}{\rho}\right)^{Cp}
\]
for a universal constant $C$.  Applying Corollary~\ref{cor:cover-det-omni} with $\rho=c\eps$ gives the above rate when such a cover is available.

To obtain the stated sample complexity without knowing $P_X$, draw an independent context sample of size $m=\widetilde O(p/\eps^2)$, or equivalently set aside this many labeled examples and ignore their labels.  Uniform convergence for the pairwise absolute-difference class $\{|d-d'|:d,d'\in\calD\}$ implies that, with high constant probability, empirical $L_1$ distances on this sample approximate their population $L_1(P_X)$ values to accuracy $O(\eps)$ uniformly over $d,d'\in\calD$.  On this event, any empirical $c'\eps$-cover is a population $c\eps$-cover after adjusting constants.  Haussler's covering bound also controls the size of such an empirical cover by $(C/\eps)^{Cp}$.  Run the finite-test learner on this cover and take constants large enough for the two success events; the cover-building sample and the learner's sample together remain within the claimed rate, and the overall success probability is at least $2/3$.
\end{proof}

\begin{corollary}[Derandomizing finite-basis offline omniprediction]\label{cor:okk-det-omni}
For every $\eps\in(0,1/10)$ and every distribution $P$ on $\calX\times\{0,1\}$ with marginal $P_X$, let $\calC_\eps$ be any finite auditor class that $(\lambda,\rho)$-represents $\Delta\calL\circ\calH$ under $P_X$ in the sense of Definition~\ref{def:finite-auditor-representation}, with $\rho\le c\eps$ for the universal constant in Theorem~\ref{thm:det-omni}.  Then the corresponding finite-basis offline omniprediction guarantee can be realized by a deterministic grid-valued predictor using
\[
        \widetilde O\left(
        \frac{(1+\lambda)^2(\log|\calC_\eps|+\log(1/\eps))}{\eps^2}
        \right)
\]
samples, with constant success probability.  In particular, instantiating $\calC_\eps$ with the finite approximate bases, spanned classes, or covers used in the offline constructions of \citet[Definitions~4.3--4.4 and Section~6]{okoroafor2025near} gives deterministic offline omnipredictors at the same (nearly) optimal statistical rate of \cite{okoroafor2025near}.
\end{corollary}

\begin{proof}
Apply Theorem~\ref{thm:det-omni} to $\calC=\calC_\eps$.  The theorem uses only the finite OI tests induced by $\calC_\eps$ and the $O(1/\eps)$ threshold representatives on the prediction grid, so the sample bound follows from $|\calA_{\rm omni}|\le|\calC_\eps|+O((1+\lambda)/\eps)$.  The final statement is obtained by taking $\calC_\eps$ to be the finite basis or cover supplied by the offline omniprediction theorem being instantiated.
\end{proof}

For example, whenever such a finite auditor class has $\log|\calC_\eps|=\widetilde O(d)$, constant coefficient norm, and approximation error $\rho=O(\eps)$ for the relevant loss/hypothesis family, Corollary~\ref{cor:okk-det-omni} gives a deterministic omnipredictor with sample complexity
\[
        \widetilde O\left(\frac{d+\log(1/\eps)}{\eps^2}\right),
\]
matching the randomized offline rate up to logarithmic factors.  The new ingredient here is only the derandomization: the finite-cover or basis construction is inherited from the omniprediction theorem being instantiated.

\section*{Acknowledgements}

We thank Jiuyao Lu for helpful comments --- in particular suggesting a simpler test time implementation, and noting that the sample on which confidence intervals are computed can also be used to compute split points, removing the need for a third split.

\bibliographystyle{plainnat}
\bibliography{refs}

\clearpage
\appendix

\section{Deferred online-to-batch proofs}\label{app:online-to-batch}

\begin{proof}[Proof of Lemma~\ref{lem:hinted-online}]
Let
\[
        \mathcal T=\calG\times\{\pm1\}^{\Lambda},
        \qquad
        M=|\mathcal T|=|\calG|2^K .
\]
Set $A=\log M+\log(3/\delta)$.  If $T<A$, the conclusion is trivial after increasing the universal constant, since each signed payoff lies in $[-1,1]$.  Assume $T\ge A$.
For a signed calibration test $r=(g,\sigma)\in\mathcal T$, write
\[
        \phi_r(x,v,m)=g(x)\sigma(v)(v-m),
        \qquad
        \phi_r(x,q,m)=\sum_{v\in\Lambda_x}q(v)\phi_r(x,v,m).
\]

We first record the one-step minimax fact.  Fix a context $x$ for which $\Lambda_x$ $\gamma$-covers $U_x$ and fix $\pi\in\Delta(\mathcal T)$.  There exists $q\in\Delta(\Lambda_x)$ such that
\begin{equation}\label{eq:minimax-one-step-app}
        \sup_{m\in U_x}\sum_{r\in\mathcal T}\pi(r)\phi_r(x,q,m)
        \le \gamma .
\end{equation}
Indeed, for each fixed $m\in U_x$, choose $a\in\Lambda_x$ with $|a-m|\le\gamma$.  Then
\[
        \sum_{r\in\mathcal T}\pi(r)\phi_r(x,a,m)
        \le
        |a-m|\sum_{r\in\mathcal T}\pi(r)
        \le \gamma .
\]
Thus $\sup_m\inf_q\sum_r\pi(r)\phi_r(x,q,m)\le\gamma$, and Sion's minimax theorem~\cite{sion1958} gives \eqref{eq:minimax-one-step-app}.  The endpoint LP computes such a $q$, because the payoff is affine in $m$.

Run Algorithm~\ref{alg:known-interval-ew}, and write $\pi_t$ for its weight distribution before round $t$.  The algorithm uses $\eta=\sqrt{A/T}\le1$.  By the minimax fact and validity of the realized hint, $q_t(X_t)$ satisfies \eqref{eq:minimax-one-step-app}.  For $r=(g,\sigma)$, let
\[
        z_t^r
        =
        g(X_t)\sum_{v\in\Lambda}q_t(X_t)(v)\sigma(v)(v-Y_t).
\]
Then $z_t^r\in[-1,1]$.  Let $W_t=\sum_r\pi_t(r)z_t^r$.  Since the conditional mean $m_t$ lies in $U_{X_t}$,
\[
        \bbE[W_t\mid X_1,Y_1,\ldots,X_{t-1},Y_{t-1},X_t]\le\gamma .
\]
Azuma--Hoeffding gives, with probability at least $1-\delta/3$,
\[
        \sum_{t=1}^T W_t
        \le
        \gamma T+C\sqrt{T\log(3/\delta)} .
\]
The exponential-weights regret bound for gains in $[-1,1]$ gives, simultaneously for every $r\in\mathcal T$,
\[
        \sum_{t=1}^T z_t^r
        \le
        \sum_{t=1}^T W_t
        +\frac{\log M}{\eta}
        +C\eta T .
\]
Combining the two bounds and the choice of $\eta$ proves
\[
        \max_{r\in\mathcal T}
        \frac1T\sum_{t=1}^T z_t^r
        \le
        \gamma+C\sqrt{\frac{\log M+\log(3/\delta)}{T}}
\]
with probability at least $1-\delta/3$.
\end{proof}

\begin{proof}[Proof of Lemma~\ref{lem:martingale-otb}]
By the definition of $Q$,
\[
        A_r(S)
        =
        \frac1T\sum_{t=1}^T
        \bbE_{(X,Y)\sim P}\left[
        g(X)\sum_{v\in\Lambda}q_t(X)(v)\sigma(v)(v-Y)
        \right].
\]
For each fixed signed calibration test, define
\[
        D_t^r=
        \bbE_{(X,Y)\sim P}\left[
        g(X)\sum_{v\in\Lambda}q_t(X)(v)\sigma(v)(v-Y)
        \right]
        -z_t^r .
\]
Since $q_t$ is determined before the fresh draw $(X_t,Y_t)$, $(D_t^r)_{t=1}^T$ is a martingale difference sequence with $|D_t^r|\le2$.  Azuma--Hoeffding and a union bound over $\mathcal T$ give the claimed uniform bound.
\end{proof}

\begin{proof}[Proof of Theorem~\ref{thm:constrained-online}]
Set $K=|\Lambda|$ and
\[
        \mathcal T=\calG\times\{\pm1\}^{\Lambda},
        \qquad
        M=|\mathcal T|=|\calG|2^K .
\]
Run Algorithm~\ref{alg:known-interval-ew} on the i.i.d.\ sample.  Since the interval-hint system is $\gamma$-valid, Lemma~\ref{lem:hinted-online} gives
\[
        \max_{r\in\mathcal T}
        \frac1T\sum_{t=1}^Tz_t^r
        \le
        \gamma+C\sqrt{\frac{\log M+\log(3/\delta)}{T}}
\]
with probability at least $1-\delta/3$.  Define the batch predictor by averaging the online rules:
\[
        Q_x(v)=\frac1T\sum_{t=1}^Tq_t(x)(v).
\]
Since $q_t(x)(v)=0$ when $v\notin\Lambda_x$, $\supp(Q_x)\subseteq\Lambda_x$ for every $x$.  Lemma~\ref{lem:martingale-otb} converts the empirical bound into the same population bound, up to another
\[
        C\sqrt{\frac{\log M+\log(3/\delta)}{T}}
\]
term.  On the intersection of these events, which has probability at least $1-\delta$, every signed calibration test has population score at most
\[
        C\left(\gamma+\sqrt{\frac{K+\log|\calG|+\log(1/\delta)}{T}}\right).
\]
Because $\mathcal T$ contains both signs of every $\sigma$, Lemma~\ref{lem:signed-form} turns this signed-test bound into $\operatorname{MC}_P(Q;\calG)$.  Finally, in Theorem~\ref{thm:constrained-online} we have $K=|\Lambda|\le C_{\rm grid}/\gamma$, giving the stated bound after adjusting the universal constant.
\end{proof}

\section{Factored implementation and query representation}\label{app:efficient}

Here we justify Algorithm~\ref{alg:known-interval-factored}.  Let
\[
        S_{g,v}^{t-1}
        =
        \sum_{s<t} g(X_s)q_s(X_s)(v)(v-Y_s)
\]
be the cumulative calibration bias for group $g$ and grid value $v$ before round $t$.  The direct exponential-weights distribution assigns signed calibration test $(g,\sigma)$ weight proportional to
\[
        \exp\left(\eta\sum_{v\in\Lambda}\sigma(v)S_{g,v}^{t-1}\right).
\]
For each fixed group,
\[
        \sum_{\sigma\in\{\pm1\}^{\Lambda}}
        \exp\left(\eta\sum_{v\in\Lambda}\sigma(v)S_{g,v}^{t-1}\right)
        =
        \prod_{v\in\Lambda}2\cosh(\eta S_{g,v}^{t-1}).
\]
Thus the marginal weight of group $g$ is
\[
        \pi_t(g)=
        \frac{\prod_{v\in\Lambda}2\cosh(\eta S_{g,v}^{t-1})}
        {\sum_{g'\in\calG}\prod_{v\in\Lambda}2\cosh(\eta S_{g',v}^{t-1})},
\]
and, conditional on $g$,
\[
        \bbE[\sigma(v)\mid g]
        =
        \tanh(\eta S_{g,v}^{t-1}).
\]
Therefore the weighted one-step objective can be computed from the coefficients
\[
        c_t(x,v)
        =
        \sum_{g\in\calG}\pi_t(g)g(x)\tanh(\eta S_{g,v}^{t-1}),
\]
which are exactly the coefficients used in Algorithm~\ref{alg:known-interval-factored}.  The same endpoint LP \eqref{eq:endpoint-lp} then computes the next distribution $q_t(x)$.

During training, the learner solves this LP only at the realized context $X_t$ and updates
\[
        S_{g,v}^{t}=S_{g,v}^{t-1}
        +
        g(X_t)q_t(X_t)(v)(v-Y_t),
        \qquad g\in\calG,\ v\in\Lambda .
\]
This uses $|\calG|K$ stored numbers rather than $|\calG|2^K$ signed-test weights.

The batch randomized predictor
\[
        Q_x=\frac1T\sum_{t=1}^T q_t(x)
\]
need not be materialized on all of $\calX$.  The learner stores the online states, the confidence-interval table from $S_0$, the sorted distinct contexts from $S_2$, and one sampler seed $(\tau_C,U_C)$ for each rounding cell.  On query $x$, it finds the singleton, gap, or cutpoint cell containing $x$, recomputes only $q_{\tau_C}(x)$ from the stored state at time $\tau_C$, and returns the inverse-CDF value for $q_{\tau_C}(x)$ at the seed $U_C$.  Since $\tau_C$ is uniform on $\{1,\ldots,T\}$, this has the same marginal law as sampling from $Q_x$, but it avoids reconstructing all $T$ distributions at query time.

\section{Concentration for learned interval hints}\label{app:interval-hints}

For the concrete constants used in the main proof, take
\[
        \gamma=\frac{\alpha}{64},
        \qquad
        N_{\rm grid}=\left\lceil\frac1\gamma\right\rceil,
        \qquad
        \Lambda=\left\{\frac{i}{N_{\rm grid}}:i=0,1,\ldots,N_{\rm grid}\right\}.
\]
Then $\Lambda$ contains $0$ and $1$, is a $\gamma$-net of $[0,1]$, and has size
\[
        K=|\Lambda|=N_{\rm grid}+1\le \frac{192}{\alpha}.
\]
We use
\[
        L=K+\log(|\calG|+1)+100,
        \qquad
        J=100\log\left(\frac{100L}{\alpha}\right),
        \qquad
        n_0=C_0\frac{LJ}{\alpha^2}.
\]

\begin{proof}[Proof of Proposition~\ref{prop:learned-intervals}]
First note that, with probability one, no non-atom appears twice in $S_0$.  Indeed, for each pair $i<j$,
\[
        \bbP\bigl(X_i=X_j\notin\At(P_X)\bigr)
        =
        \bbE\left[\one\{X_i\notin\At(P_X)\}P_X(\{X_i\})\right]
        =
        0,
\]
and there are only finitely many pairs.  We include this full-probability event in $\mathcal E_{\rm cov}$.

Condition on any multiset of contexts with this no-repeated-non-atom property.  If $N_x\ge2$, then $x$ is an atom, and the labels with context $x$ are independent $[0,1]$-valued variables with mean $\mu(x)$.  Hoeffding's inequality gives
\[
        \bbP\left(|\widehat\mu_x-\mu(x)|>\sqrt{\frac{J}{N_x}}\ \middle|\ X_1,\ldots,X_{n_0}\right)
        \le 2e^{-2J}.
\]
There are at most $n_0$ contexts with $N_x\ge2$.  Since $J$ is a sufficiently large multiple of $\log(100L/\alpha)$ and $n_0\le C LJ/\alpha^2$, the union bound gives failure probability at most $0.01$.  This defines $\mathcal E_{\rm cov}$.

On $\mathcal E_{\rm cov}$, if $N_x\ge2$ then $\mu(x)\in I_x$.  If $N_x=0$, then $I_x=[0,1]$ by definition.  If $N_x=1$, then $r_x=1$ because $J\ge1$, and $\widehat\mu_x\in[0,1]$, so again $I_x=[0,1]$.  Thus $\mu(x)\in I_x$ for every $x$.  If $m\in I_x$, then the $\gamma$-net property gives $a\in\Lambda$ with $|a-m|\le\gamma$.  By definition of $\Lambda_x$, this $a$ belongs to $\Lambda_x$.  Therefore the learned interval-hint system is $\gamma$-valid.

It remains to prove the radius event.  It suffices to show
\[
        \bbE_{S_0}\left[\sum_{x\in\At(P_X)}p_x^2r_x^2\right]
        \le
        C\frac{\alpha^2}{L},
\]
because Markov's inequality then gives an event $\mathcal E_{\rm rad}$ of probability at least $0.90$ on which this same sum is at most $C_1\alpha^2/L$, after increasing the universal constant $C_1$.  All summands are nonnegative, so we may sum one-point expectation bounds over the countable atom set.

Fix an atom $x$ with mass $p=p_x$ and let $N=N_x\sim\operatorname{Bin}(n_0,p)$.  Define
\[
        r(N)=\begin{cases}
        1,&N=0,\\
        \min\{1,\sqrt{J/N}\},&N\ge1.
        \end{cases}
\]
We first show
\begin{equation}\label{eq:single-rad}
        \bbE[p^2 r(N)^2]
        \le
        C\frac{J}{n_0}p .
\end{equation}
Let $\lambda=n_0p$.  If $\lambda\le 2J$, then $p^2\le (2J/n_0)p$, and \eqref{eq:single-rad} follows because $r(N)^2\le1$.  If $\lambda>2J$, then
\[
        \bbE[r(N)^2]
        \le
        \bbP(N<\lambda/2)+\frac{2J}{\lambda}.
\]
By a Chernoff bound, $\bbP(N<\lambda/2)\le e^{-\lambda/8}$.  Since $\lambda e^{-\lambda/8}$ is uniformly bounded for $\lambda>0$,
\[
        p^2 e^{-\lambda/8}
        =
        \frac{p}{n_0}\lambda e^{-\lambda/8}
        \le
        C\frac{p}{n_0}.
\]
Also $p^2(2J/\lambda)=2Jp/n_0$.  Thus \eqref{eq:single-rad} holds in all cases.  Summing \eqref{eq:single-rad} over the atoms gives
\[
        \bbE\sum_{x\in\At(P_X)}p_x^2r_x^2
        \le
        C\frac{J}{n_0}
        \le
        C\frac{\alpha^2}{L},
\]
using $n_0=C_0LJ/\alpha^2$.  Markov's inequality gives $\mathcal E_{\rm rad}$, and the final probability statement follows by the union bound.
\end{proof}

\section{Deferred proof for finite OI tests}\label{app:finite-oi-proof}

\begin{proof}[Proof of Lemma~\ref{lem:finite-test-otb}]
Let
\[
        \mathcal T=\calA_\pm=\{a,-a:a\in\calA\},
        \qquad
        M=|\mathcal T|\le 2|\calA|.
\]
Run exponential weights over $\mathcal T$.  At round $t$, after the past history has fixed the distribution $\pi_t\in\Delta(\mathcal T)$, set
\[
        c_t(x,v)=\sum_{a\in\mathcal T}\pi_t(a)a(x,v),
        \qquad v\in\Lambda ,
\]
and choose $q_t(x)=\operatorname{LP}_x(c_t(x,\cdot))$.

The one-step minimax step is identical to the one in Lemma~\ref{lem:hinted-online}.  Fix a context $x$ where the hint is valid and a distribution $\pi\in\Delta(\mathcal T)$.  For every $m\in U_x$, choose $v_m\in\Lambda_x$ with $|v_m-m|\le\gamma$.  Since every test is bounded by one,
\[
        \sum_{a\in\mathcal T}\pi(a)a(x,v_m)(v_m-m)
        \le \gamma .
\]
Sion's minimax theorem~\cite{sion1958} gives an action $q\in\Delta(\Lambda_x)$ whose expected payoff is at most $\gamma$ for every mean value in $U_x$, and the endpoint LP computes such an action.

For $a\in\mathcal T$, write
\[
        z_t^a=
        \sum_{v\in\Lambda}q_t(X_t)(v)a(X_t,v)(v-Y_t)\in[-1,1].
\]
Let $W_t=\sum_{a\in\mathcal T}\pi_t(a)z_t^a$.  Since $\mu(X_t)\in U_{X_t}$ and $q_t(X_t)$ satisfies the one-step minimax bound,
\[
        \bbE[W_t\mid X_1,Y_1,\ldots,X_{t-1},Y_{t-1},X_t]\le \gamma .
\]
Azuma--Hoeffding gives, with probability at least $1-\delta/4$,
\[
        \sum_{t=1}^T W_t
        \le
        \gamma T+C\sqrt{T\log(1/\delta)} .
\]
The standard exponential-weights regret bound gives, simultaneously for every $a\in\mathcal T$,
\[
        \sum_{t=1}^T z_t^a
        \le
        \sum_{t=1}^T W_t
        +C\sqrt{T\log M}.
\]
Thus, with probability at least $1-\delta/2$,
\[
        \max_{a\in\mathcal T}
        \frac1T\sum_{t=1}^T z_t^a
        \le
        \gamma+
        C\sqrt{\frac{\log M+\log(1/\delta)}{T}} .
\]

Define the averaged batch predictor $Q_x(v)=T^{-1}\sum_tq_t(x)(v)$.  For each $a\in\mathcal T$, the difference between
\[
        \bbE\left[
        \sum_{v\in\Lambda}Q_X(v)a(X,v)(v-Y)
        \right]
        \quad\text{and}\quad
        \frac1T\sum_{t=1}^Tz_t^a
\]
is an average of martingale differences bounded by $2$.  Azuma--Hoeffding and a union bound over $\mathcal T$ control all these differences by
\[
        C\sqrt{\frac{\log M+\log(1/\delta)}{T}}
\]
with probability at least $1-\delta/2$.  On the intersection of the two events, the population residual is bounded by the claimed quantity for every signed test in $\mathcal T$.  Since $\mathcal T$ contains both signs of each test in $\calA$, this is exactly $\operatorname{OIErr}_P(Q;\calA)$.  Finally, $M\le2|\calA|$ and $\supp(Q_x)\subseteq\Lambda_x$ because each online action has this support.
\end{proof}

\section{Deterministic panprediction from OI tests}\label{app:panprediction}

We use the notation of \citet{balakrishnan2025panprediction}.  Outcomes are binary, $Y\in\{0,1\}$.  A group is a binary function $g:\calX\to\{0,1\}$, and
\[
        P_g=\bbP[g(X)=1].
\]
As in Section~\ref{sec:linear-tests}, fix a Bayes-act choice for each loss.  For a loss $\ell:[0,1]\times\{0,1\}\to[-1,1]$, write
\[
        k_\ell(p)\in
        \operatorname*{argmin}_{a\in[0,1]}
        \bbE_{Z\sim{\rm Bernoulli}(p)}[\ell(a,Z)]
\]
using the chosen act when there are multiple minimizers.  A deterministic predictor $p:\calX\to[0,1]$ is an $(\calL,\calG,\calH,\eps)$-panpredictor if, for every $\ell\in\calL$ and every $g\in\calG$ with $P_g>0$,
\[
        \bbE[\ell(k_\ell(p(X)),Y)\mid g(X)=1]
        \le
        \inf_{f\in\calH}
        \bbE[\ell(f(X),Y)\mid g(X)=1]
        +\frac{\eps}{\sqrt{P_g}} .
\]
As in \citet{balakrishnan2025panprediction}, we consider losses of bounded variation in the first argument, normalized so that their threshold-basis approximation has universal constants.

\begin{definition}[Step calibration]\label{def:app-step-calibration}
Let $\calG$ be a group family and let $\calH\subseteq[0,1]^\calX$.  A deterministic predictor $p:\calX\to[0,1]$ is $(\calG,\calH,\eta)$-step calibrated if, for every $g\in\calG$, $f\in\calH$, and $v,w\in[0,1]$,
\[
        \left|
        \bbE[(Y-p(X))\one\{p(X)\le v,\ f(X)\le w\}\mid g(X)=1]
        \right|
        \le
        \frac{\eta}{\sqrt{P_g}} .
\]
\end{definition}

\begin{lemma}[Step calibration implies panprediction]\label{lem:app-step-to-pan}
For bounded-variation losses, every deterministic $(\calG,\calH,\eta)$-step calibrated predictor is an $(\calL,\calG,\calH,C\eta)$-panpredictor, for a universal constant $C$.
\end{lemma}

\begin{proof}
This is the deterministic reduction of \citet[Theorem~3.3]{balakrishnan2025panprediction}.  We recall the short argument.  Applied pointwise with $q=p(X)$, the postprocessing $k_\ell(q)$ is optimal under a Bernoulli label with mean $q$, so for every comparator $f$,
\[
        \bbE_{Z\sim{\rm Bernoulli}(p(X))}
        [\ell(k_\ell(p(X)),Z)]
        \le
        \bbE_{Z\sim{\rm Bernoulli}(p(X))}
        [\ell(f(X),Z)]
\]
pointwise in $X$.  It remains to compare these predictor-induced label expectations to the true label expectation, group by group.  Writing $\Delta\ell(a)=\ell(a,1)-\ell(a,0)$, the two discrepancies are
\[
        \bbE[(Y-p(X))\Delta\ell(k_\ell(p(X)))\mid g(X)=1]
\]
and
\[
        \bbE[(Y-p(X))\Delta\ell(f(X))\mid g(X)=1].
\]
The first discrepancy is controlled by the monotonicity argument of Lemma~\ref{lem:threshold-bayes-act}, applied under the conditional distribution given $g(X)=1$.  After rescaling the loss to $[0,1]$, that argument shows that
\[
        u\mapsto \Delta\ell(k_\ell(u))
\]
is a monotone function of $u$ whose range is bounded by a universal constant, and hence has a threshold-basis representation with universal coefficient norm.  Definition~\ref{def:app-step-calibration} controls those threshold terms by using thresholds of $p$ and setting the comparator threshold to $1$.

The second discrepancy is different: it involves $\Delta\ell(f(X))$, not the Bayes-act map.  Here we use the bounded-variation normalization directly, which gives a threshold-basis representation with universal coefficient norm for the action-loss difference
\[
        u\mapsto\Delta\ell(u).
\]
Definition~\ref{def:app-step-calibration} controls these terms by setting the prediction threshold to $1$ and using thresholds of $f$.  Thus each discrepancy is at most $O(\eta/\sqrt{P_g})$, and substituting these two comparisons into the ideal predictor-induced inequality gives the claimed panprediction guarantee.
\end{proof}

\begin{theorem}[Deterministic finite-class panprediction]\label{thm:app-panprediction}
Let $\calL$ be a class of bounded-variation losses, let $\calG$ be a finite family of binary groups, and let $\calH\subseteq\Gamma^\calX$ be a finite nonempty class of predictors taking values in a finite grid $\Gamma\subseteq[0,1]$.  Fix a distribution $P$ on $\calX\times\{0,1\}$ and suppose
\[
        p_{\min}=\min_{g\in\calG}P_g>0 .
\]
Assume the values $P_g$ are supplied to the learner.  There is a universal constant $C$ such that, for every $\eps\in(0,1/10)$, there is an algorithm using
\[
        n
        \le
        C\frac{L_{\rm pan}\log(100L_{\rm pan}/(\eps\sqrt{p_{\min}}))}
        {\eps^2 p_{\min}}
\]
samples that outputs a deterministic predictor $p:\calX\to[0,1]$ which is an $(\calL,\calG,\calH,\eps)$-panpredictor with probability at least $2/3$, where
\[
        L_{\rm pan}
        =
        \log\left(
        2|\calG||\calH|(|\Gamma|+1)
        \left(\frac{C}{\eps\sqrt{p_{\min}}}+1\right)
        +2
        \right)+100 .
\]
In particular, if $p_{\min}$ is constant and $|\Gamma|$, $|\calG|$, and $|\calH|$ are polynomially bounded in $1/\eps$, then $n=\widetilde O(\eps^{-2})$.
\end{theorem}

\begin{proof}
Let $\tau=c_0\eps\sqrt{p_{\min}}$ for a sufficiently small universal constant $c_0$.  Run the finite-test OI learner of Theorem~\ref{thm:linear-tests} with target accuracy $\tau$.  Let $\Lambda$ be the prediction grid used by that learner, and let $\Theta_\Lambda$ and $\Theta_\Gamma$ contain one representative threshold for each distinct sublevel-set pattern on $\Lambda$ and $\Gamma$, respectively.  Thus
\[
        |\Theta_\Lambda|\le|\Lambda|+1=O(1/\tau),
        \qquad
        |\Theta_\Gamma|\le|\Gamma|+1 .
\]
Use the OI test family
\[
        \calA_{\rm pan}
        =
        \left\{
        (x,z)\mapsto
        \sqrt{\frac{p_{\min}}{P_g}}\,
        g(x)\one\{z\le \vartheta\}\one\{f(x)\le \omega\}:
        g\in\calG,\ f\in\calH,\ \vartheta\in\Theta_\Lambda,\ \omega\in\Theta_\Gamma
        \right\}.
\]
The factor $\sqrt{p_{\min}/P_g}$ converts a uniform OI error into the group-conditional $1/\sqrt{P_g}$ tolerance in Definition~\ref{def:app-step-calibration}, while keeping every test bounded by one.  Theorem~\ref{thm:linear-tests} therefore gives, with probability at least $2/3$, a deterministic grid predictor $p:\calX\to\Lambda$ satisfying
\[
        \left|
        \bbE\left[
        \sqrt{\frac{p_{\min}}{P_g}}\,
        g(X)\one\{p(X)\le \vartheta\}\one\{f(X)\le \omega\}(p(X)-Y)
        \right]
        \right|
        \le \tau
\]
simultaneously for all choices of $g,f,\vartheta,\omega$ above.  Since $p$ is $\Lambda$-valued and every $f\in\calH$ is $\Gamma$-valued, the same inequality holds for arbitrary thresholds $v,w\in[0,1]$ after replacing them by their representatives.  Rearranging and dividing by $P_g$ yields
\[
        \left|
        \bbE[(Y-p(X))\one\{p(X)\le v,\ f(X)\le w\}\mid g(X)=1]
        \right|
        \le
        \frac{\tau}{\sqrt{p_{\min}}\sqrt{P_g}}
        =
        \frac{c_0\eps}{\sqrt{P_g}} .
\]
Thus $p$ is $(\calG,\calH,c_0\eps)$-step calibrated.  Lemma~\ref{lem:app-step-to-pan} implies that $p$ is an $(\calL,\calG,\calH,Cc_0\eps)$-panpredictor, and choosing $c_0$ small enough gives the stated error.

Finally,
\[
        |\calA_{\rm pan}|
        \le
        |\calG||\calH|(|\Gamma|+1)(|\Lambda|+1)
        \le
        |\calG||\calH|(|\Gamma|+1)
        \left(\frac{C}{\eps\sqrt{p_{\min}}}+1\right),
\]
and substituting this bound into Theorem~\ref{thm:linear-tests} with target accuracy $\tau$ gives the stated sample complexity.
\end{proof}

\begin{remark}[Finitely covered classes]\label{rem:app-pan-covers}
The finite-class assumption can be replaced by a finite-cover assumption in the same way as in Corollary~\ref{cor:cover-det-omni}.  The relevant class is the normalized threshold-conjunction class
\[
        \left\{
        x\mapsto
        \sqrt{\frac{p_{\min}}{P_g}}\,
        g(x)\one\{f(x)\le w\}:
        g\in\calG,\ f\in\calH,\ w\in[0,1]
        \right\}.
\]
Here $p_{\min}=\min_{g\in\calG}P_g$, as in Theorem~\ref{thm:app-panprediction}.
If a finite $L_1(P_X)$ cover of this class at radius $O(\eps\sqrt{p_{\min}})$ is supplied, the same proof gives the corresponding logarithmic dependence on the cover size.  For VC group classes and finite-pseudo-dimension benchmark classes, such covers can be built from an independent context sample by the usual empirical-cover argument, at the same $\widetilde O(\eps^{-2}p_{\min}^{-1})$ scale. 
\end{remark}

\section{Removing training randomness}\label{app:training-derandomization}

The algorithms in the body of the paper output prediction-time deterministic predictors, but they are presented with one source of training randomness: the final cell sampler seeds used for rounding.  This randomness can also be removed.  We state the argument for finite OI test families, since multicalibration is obtained by taking the signed calibration tests and the omniprediction applications use finite OI tests after the cover or basis reduction.  The derandomization below is information-theoretic; it is not meant to preserve computational efficiency.

The idea is simple.  Rather than picking the rounding-cell seeds randomly, we enumerate a smaller explicit family of candidate seed assignments, evaluate the resulting predictors on a holdout set, and then pick the best one.  The idea on its own fails if we enumerate all fully independent seed assignments: the candidate family would be far too large for uniform convergence on a holdout set of the target size.  The fix is to use seeds that are not truly independent, but only $k$-wise independent.  Taking $k$ on the order of $\log|\calA|$ (and the desired failure-probability parameter) is sufficient for the same rounding concentration argument, without exceeding our target sample complexity.

Fix the setup of Proposition~\ref{prop:finite-test-rounding}.  Let $\Pi$ be the finite family of rounding cells, let $M=|\Pi|$, and enumerate $\Lambda=\{v_1,\ldots,v_K\}$.  For each $x$, let
\[
        F_x(u)=v_j
        \quad\text{when}\quad
        \sum_{i<j}Q_x(v_i)<u\le \sum_{i\le j}Q_x(v_i),
\]
with the endpoint convention $F_x(0)=v_1$.

\begin{lemma}[Limited-independent cell rounding]\label{lem:kwise-rounding}
There is a universal constant $C$ such that the following holds.  Let $\beta,\eta\in(0,1)$, and let $q$ be a prime with
\[
        q\ge \max\{M,\,4K/\eta\}.
\]
Assign distinct identifiers $\iota(C)\in\mathbb F_q$ to the cells $C\in\Pi$.  Let $k$ be an even integer with
\[
        k\ge C\log\frac{2|\calA|}{\beta}.
\]
For each polynomial $p\in\mathbb F_q[z]$ of degree less than $k$, define a deterministic predictor
\[
        h_p(x)=F_x\left(\frac{p(\iota(C))+1/2}{q}\right)
        \qquad\text{for }x\in C,
\]
where the field element $p(\iota(C))$ is identified with its representative in $\{0,\ldots,q-1\}$.  Then at least a $1-\beta$ fraction of these polynomials satisfy
\[
\begin{aligned}
&\max_{a\in\calA}
\Bigg|
\bbE\bigl[a(X,h_p(X))(h_p(X)-Y)\bigr]
\\
&\qquad -
\bbE\left[\sum_{v\in\Lambda}Q_X(v)a(X,v)(v-Y)\right]
\Bigg|
\\
&\le
C\left(
        \sqrt{V_\Pi\log\frac{2|\calA|}{\beta}}
        +\gamma+\eta
\right).
\end{aligned}
\]
\end{lemma}

\begin{proof}
Choose $p$ uniformly from the $q^k$ polynomials of degree less than $k$.  For any $k$ distinct cells, the values $p(\iota(C))$ are independent and uniform in $\mathbb F_q$, so the induced cell seeds are $k$-wise independent and each is uniform on $\{(j+1/2)/q:j=0,\ldots,q-1\}$.  Using midpoints avoids putting mass on the artificial endpoint $u=0$ in the inverse-CDF convention.

First compare this discrete seed to a continuous uniform seed.  If $J$ is uniform on $\{0,\ldots,q-1\}$ and $\widetilde Q_x$ is the distribution of $F_x((J+1/2)/q)$, then for every function $\psi:\Lambda\to[-1,1]$,
\begin{equation}\label{eq:discrete-seed-bias}
        \left|
        \sum_{v\in\Lambda}\widetilde Q_x(v)\psi(v)
        -
        \sum_{v\in\Lambda}Q_x(v)\psi(v)
        \right|
        \le \frac{2K}{q}\le \frac{\eta}{2}.
\end{equation}
Indeed, the probability of each inverse-CDF interval is approximated by the fraction of midpoints it contains, and the sum of the absolute errors over the $K$ intervals is at most $2K/q$.

Write $I_x=[a_x,b_x]$ and define $\clip_x(v)=\min\{\max\{v,a_x\},b_x\}$, as in Proposition~\ref{prop:finite-test-rounding}.  Fix a test $a\in\calA$ and define the clipped payoff
\[
        \psi_{a,x}(v)=a(x,v)(\clip_x(v)-\mu(x)).
\]
For a cell $C$, set
\[
        Z_{a,C}
        =
        \int_C
        \left(
        \psi_{a,x}\left(F_x\left(\frac{p(\iota(C))+1/2}{q}\right)\right)
        -
        \frac1q\sum_{j=0}^{q-1}\psi_{a,x}\left(F_x\left(\frac{j+1/2}{q}\right)\right)
        \right)\,dP_X(x).
\]
The variables $\{Z_{a,C}\}_{C\in\Pi}$ are mean zero and $k$-wise independent.  Their ranges are bounded exactly as in the proof of Proposition~\ref{prop:finite-test-rounding}: if $C=\{x\}$ is an observed singleton, then $|Z_{a,C}|\le 4p_xr_x$ when $p_x>0$ and $Z_{a,C}=0$ when $p_x=0$, while if $C\in\Pi_{\rm unobs}$ then $|Z_{a,C}|\le 2P_X(C)$.  Hence
\[
        \sum_{C\in\Pi}\|Z_{a,C}\|_\infty^2\le C_0 V_\Pi
\]
for a universal constant $C_0$.

We use the standard Chernoff--Hoeffding moment bound for bounded $k$-wise independent sums~\citep{schmidt1995chernoff}: if $Z_C$ are mean-zero, $k$-wise independent, $|Z_C|\le b_C$, and $B^2=\sum_C b_C^2$, then for even $k$,
\[
        \bbP\left(\left|\sum_C Z_C\right|> C_1B\sqrt{k}\right)\le \exp(-k).
\]
For the simple form used here, this follows directly from the usual $k$th-moment proof for independent bounded sums, since the $k$th moment depends only on joint distributions of at most $k$ coordinates.  Applying this with $B^2\le C_0V_\Pi$ and our choice of $k$ gives, for the fixed test $a$,
\[
        \left|\sum_C Z_{a,C}\right|
        \le
        C_2\sqrt{V_\Pi\log\frac{2|\calA|}{\beta}}
\]
except with probability at most $\beta/|\calA|$.  A union bound controls all tests simultaneously.

It remains to add back the two deterministic approximation terms.  Equation~\eqref{eq:discrete-seed-bias}, applied pointwise and integrated over $x$, contributes at most $\eta/2$ to the clipped payoff.  The unclipped grid-rounding part contributes at most $2\gamma$, exactly as in Proposition~\ref{prop:finite-test-rounding}: the midpoint convention ensures $h_p(x)\in\supp(Q_x)\subseteq\Lambda_x$, and all grid values in $\Lambda_x$ lie within distance $\gamma$ of $I_x$.  Absorbing constants proves the lemma.
\end{proof}

\begin{lemma}[Validation over a finite candidate list]\label{lem:validation-selection}
Let $\calH$ be any finite collection of deterministic predictors $h:\calX\to\Lambda$, and let $S_{\rm val}$ be an independent validation sample of size
\[
        N_{\rm val}
        \ge
        C\frac{\log(2|\calA||\calH|/\delta)}{\eta^2}.
\]
For $h\in\calH$, define the empirical OI error
\[
        \widehat{\operatorname{OIErr}}_{S_{\rm val}}(h;\calA)
        =
        \max_{a\in\calA}
        \left|
        \frac1{N_{\rm val}}\sum_{(X_i,Y_i)\in S_{\rm val}}
        a(X_i,h(X_i))(h(X_i)-Y_i)
        \right|.
\]
With probability at least $1-\delta$ over $S_{\rm val}$, every $h\in\calH$ satisfies
\[
        \left|
        \widehat{\operatorname{OIErr}}_{S_{\rm val}}(h;\calA)
        -
        \operatorname{OIErr}_{P}(h;\calA)
        \right|
        \le \eta .
\]
Consequently, if some $h^\star\in\calH$ has $\operatorname{OIErr}_{P}(h^\star;\calA)\le\rho$, then any empirical minimizer $\widehat h$ has
\[
        \operatorname{OIErr}_{P}(\widehat h;\calA)\le \rho+2\eta .
\]
\end{lemma}

\begin{proof}
For fixed $h$ and $a$, the summands are in $[-1,1]$.  Hoeffding's inequality and a union bound over $\calA\times\calH$ give the uniform deviation event.  On this event, an empirical minimizer $\widehat h$ satisfies
\[
        \operatorname{OIErr}_{P}(\widehat h;\calA)
        \le
        \widehat{\operatorname{OIErr}}_{S_{\rm val}}(\widehat h;\calA)+\eta
        \le
        \widehat{\operatorname{OIErr}}_{S_{\rm val}}(h^\star;\calA)+\eta
        \le
        \rho+2\eta .
\]
\end{proof}

\begin{theorem}[Deterministic training for finite OI tests]\label{thm:deterministic-training-oi}
There is a universal constant $C$ such that the following holds.  Let $\calA\subseteq[-1,1]^{\calX\times[0,1]}$ be a finite nonempty test family, let
\[
        L_{\calA}=\log(2|\calA|+2)+100,
\]
and let $\eps\in(0,1/10)$.  There is a deterministic learning algorithm which uses
\[
        n
        \le
        C\frac{L_{\calA}\log(C L_{\calA}/\eps)}{\eps^2}
\]
i.i.d.\ samples from $P$, outputs a deterministic grid predictor $h:\calX\to\Lambda$ satisfying
\[
        \operatorname{OIErr}_{P}(h;\calA)\le\eps
\]
with probability at least $2/3$ over the sample.  The same conclusion therefore applies to the finite-test multicalibration and omniprediction instantiations used in the body, with the same sample bounds up to the logarithmic factors already hidden in the $\widetilde O(\cdot)$ notation.
\end{theorem}

\begin{proof}
Run the construction from Theorem~\ref{thm:linear-tests} at internal accuracy $\alpha=c\eps$, but do not draw the final independent rounding seeds.  Instead, reserve an independent validation sample $S_{\rm val}$; all earlier samples are used exactly as in Theorem~\ref{thm:linear-tests} to construct the interval hints, the rounding cells, and the randomized grid predictor $Q$.

From the proof of Theorem~\ref{thm:linear-tests}, with probability at least $0.75$ over these non-validation samples, the following three conclusions hold simultaneously:
\[
        \operatorname{OIErr}_{P}(Q;\calA)\le C_1\alpha,
        \qquad
        \supp(Q_x)\subseteq\Lambda_x\quad\text{for every }x,
        \qquad
        V_\Pi\le C_1\frac{\alpha^2}{L_{\calA}}.
\]
Condition on this event.  Let $M=|\Pi|$ and choose a prime
\[
        q\le C_2\left(M+\frac{K}{\alpha}+10\right)
        \quad\text{with}\quad
        q\ge \max\{M,4K/\alpha\},
\]
which exists by Bertrand's postulate after increasing $C_2$.  The number of cells is at most the number of observed singletons plus the number of unobserved cells, so $M\le n_0+m+1$; together with $K=O(1/\alpha)$ and the sample sizes in Theorem~\ref{thm:linear-tests}, this gives $\log q\le C\log(C L_{\calA}/\alpha)$.  Let $k$ be the smallest even integer at least $C_3L_{\calA}$, and let $\calH$ be the set of predictors $h_p$ obtained from all degree-$(k-1)$ polynomials over $\mathbb F_q$ as in Lemma~\ref{lem:kwise-rounding}.  Then
\[
        |\calH|=q^k,
        \qquad
        \log|\calH|\le C_4L_{\calA}\log(C_4L_{\calA}/\alpha).
\]
Lemma~\ref{lem:kwise-rounding}, applied with $\beta=1/4$ and $\eta=\alpha$, implies that at least one $h^\star\in\calH$ satisfies
\[
        \operatorname{OIErr}_{P}(h^\star;\calA)
        \le
        \operatorname{OIErr}_{P}(Q;\calA)
        +
        C_5\left(
        \sqrt{\frac{\alpha^2}{L_{\calA}}\log(8|\calA|)}
        +\gamma+\alpha
        \right)
        \le C_6\alpha .
\]

Now use the validation sample to choose an empirical minimizer of $\widehat{\operatorname{OIErr}}_{S_{\rm val}}(h;\calA)$ over $h\in\calH$.  By Lemma~\ref{lem:validation-selection}, a validation sample of size
\[
        N_{\rm val}
        \ge
        C_7\frac{L_{\calA}\log(C_7L_{\calA}/\alpha)}{\alpha^2}
\]
is enough to ensure, with probability at least $0.95$ over $S_{\rm val}$, that the selected predictor $\widehat h$ satisfies
\[
        \operatorname{OIErr}_{P}(\widehat h;\calA)
        \le C_6\alpha+2\alpha.
\]
Choosing the universal constant $c$ in $\alpha=c\eps$ small enough makes this at most $\eps$.

The total success probability is at least $0.75\cdot0.95>2/3$.  The base construction and validation sample sizes are both bounded by
\[
        C\frac{L_{\calA}\log(C L_{\calA}/\eps)}{\eps^2}
\]
after adjusting the universal constant $C$.  The algorithm is deterministic because the polynomial list is enumerated in a fixed order, the validation minimizer is tie-broken deterministically, and all LP tie-breaking in the earlier construction was fixed deterministically.
\end{proof}

\end{document}